\documentclass[letterpaper]{article}    % DO NOT CHANGE THIS
\usepackage{aaai2026}       % DO NOT CHANGE THIS

\usepackage{times}                      % DO NOT CHANGE THIS
\usepackage{helvet}                     % DO NOT CHANGE THIS
\usepackage{courier}                    % DO NOT CHANGE THIS
\usepackage[hyphens]{url}               % DO NOT CHANGE THIS
\usepackage{graphicx}                   % DO NOT CHANGE THIS
\usepackage{natbib}     % DO NOT CHANGE THIS AND DO NOT ADD ANY OPTIONS TO IT
\usepackage{caption}    % DO NOT CHANGE THIS AND DO NOT ADD ANY OPTIONS TO IT
\usepackage{amsmath,amsfonts,amssymb,amsthm} % 基础数学包
\usepackage{booktabs}
\usepackage{threeparttable}
\usepackage{subfigure}
\usepackage{multirow}
\usepackage{bm} % 加粗符号

\usepackage{color}
\usepackage{colortbl}
\usepackage{xcolor}
\usepackage{tcolorbox}

\newtheoremstyle{mydefn}
{}{}
{\it}       % body font
{0pt}       % indent
{\bfseries} % head font
{:~}        % punctuation after head
{0.25em}    % spacing after head
{}          % CUSTOM-HEAD-SPEC
\theoremstyle{mydefn}

\newtheorem{definition}{Definition}[section]
\newtheorem{proposition}{Proposition}[section]

\newtheorem{theorem}{Theorem}[section]

\newtheoremstyle{myexample}
{}{}
{}          % body font
{0pt}       % indent
{\bfseries} % head font
{:~}        % punctuation after head
{0.25em}    % spacing after head
{}          % CUSTOM-HEAD-SPEC
\theoremstyle{myexample}

\newtheorem{example}{Example}[section]

\makeatletter
\newif\if@restonecol
\makeatother

\usepackage[linesnumbered,ruled,vlined]{algorithm2e}
\usepackage{algpseudocode}

\SetKwComment{Comment}{$\triangleright$\ }{}
\SetKwRepeat{Do}{do}{while}

\renewcommand{\paragraph}[1]{\smallskip\noindent\textbf{#1.}}
\renewcommand{\subparagraph}[1]{\smallskip\noindent\textbf{\underline{#1.}}}

\renewenvironment{proof}{\noindent{\bfseries Proof:}}{\qed \smallskip}

\DeclareMathOperator*{\argmin}{arg\,min}

\newcommand{\ip}[1]{\langle #1 \rangle}
\newcommand{\Norm}[1]{\left\lVert #1 \right\rVert}

\def\Vec#1{{\boldsymbol{#1}}}
\def\Space#1{{\mathbb{#1}}}
\def\Set#1{{\mathcal{#1}}}
\def\tVec#1{{\tilde{\boldsymbol{#1}}}}

\def\Mat#1{{\boldsymbol{#1}}}
\def\Size#1{\left| #1 \right|}
\def\term#1{\texttt{#1}}

\newtcolorbox{graybox}{
  colback=gray!10,
  colframe=gray!50,
  boxrule=1pt,
  arc=3pt,
  left=5pt,
  right=5pt,
  top=2pt,
  bottom=0pt
}

\title{One Swallow Does Not Make a Summer: Understanding Semantic Structures in Embedding Spaces}

\author{
    Yandong Sun,\textsuperscript{\rm 1}
    Qiang Huang,\textsuperscript{\rm 2}
    Ziwei Xu,\textsuperscript{\rm 1}
    Yiqun Sun,\textsuperscript{\rm 1}
    Yixuan Tang,\textsuperscript{\rm 1}
    Anthony K. H. Tung\textsuperscript{\rm 1}
}
\affiliations{
    \textsuperscript{\rm 1}School of Computing, National University of Singapore, Singapore \\
    \textsuperscript{\rm 2}School of Intelligence Science and Engineering, Harbin Institute of Technology (Shenzhen), Shenzhen, China \\
    
    \{sun.yandong, ziwei.xu\}@u.nus.edu, huangqiang@hit.edu.cn, \{sunyq, yixuan, atung\}@comp.nus.edu.sg
}

\begin{document}
\maketitle

\begin{abstract}
Embedding spaces are fundamental to modern AI, translating raw data into high-dimensional vectors that encode rich semantic relationships.  
%%%
Yet, their internal structures remain opaque, with existing approaches often sacrificing semantic coherence for structural regularity or incurring high computational overhead to improve interpretability.  
%%%
To address these challenges, we introduce the \textbf{Semantic Field Subspace (SFS)}, a geometry-preserving, context-aware representation that captures local semantic neighborhoods within the embedding space.
We also propose \textbf{SAFARI} (\textbf{S}em\textbf{A}ntic \textbf{F}ield subsp\textbf{A}ce dete\textbf{R}m\textbf{I}nation), an unsupervised, modality-agnostic algorithm that uncovers hierarchical semantic structures using a novel metric called Semantic Shift, which quantifies how semantics evolve as \textbf{SFSes} evolve.  
%%%
To ensure scalability, we develop an efficient approximation of Semantic Shift that replaces costly SVD computations, achieving a 15$\sim$30$\times$ speedup with average errors below 0.01.  
%%%
Extensive evaluations across six real-world text and image datasets show that \textbf{SFSes} outperform standard classifiers not only in classification but also in nuanced tasks such as political bias detection, while \textbf{SAFARI} consistently reveals interpretable and generalizable semantic hierarchies.  
%%%
This work presents a unified framework for structuring, analyzing, and scaling semantic understanding in embedding spaces.
\end{abstract}

% Embedding spaces are fundamental to modern AI, translating raw data into high-dimensional vectors that encode rich semantic relationships. Yet, their internal structures remain opaque, with existing approaches often sacrificing semantic coherence for structural regularity or incurring high computational overhead to improve interpretability. To address these challenges, we introduce the \textbf{Semantic Field Subspace (SFS)}, a geometry-preserving, context-aware representation that captures local semantic neighborhoods within the embedding space. We also propose \textbf{SAFARI} (\textbf{S}em\textbf{A}ntic \textbf{F}ield subsp\textbf{A}ce dete\textbf{R}m\textbf{I}nation), an unsupervised, modality-agnostic algorithm that uncovers hierarchical semantic structures using a novel metric called Semantic Shift, which quantifies how semantics evolve as \textbf{SFSes} evolve. To ensure scalability, we develop an efficient approximation of Semantic Shift that replaces costly SVD computations, achieving a 15$\sim$30$\times$ speedup with average errors below 0.01. Extensive evaluations across six real-world text and image datasets show that \textbf{SFSes} outperform standard classifiers not only in classification but also in nuanced tasks such as political bias detection, while \textbf{SAFARI} consistently reveals interpretable and generalizable semantic hierarchies. This work presents a unified framework for structuring, analyzing, and scaling semantic understanding in embedding spaces.
\section{Introduction}
\label{sect:intro}

%%% importance of embedding spaces 
Embedding spaces are foundational to modern AI systems, converting unstructured inputs--such as text, images, and time series--into dense vectors that capture semantic properties in a tractable format.
By translating semantic similarity into geometric proximity,  these spaces enable efficient comparison, retrieval, and manipulation across modalities.
%%% applications (evidences)
This property supports a wide range of applications, including knowledge retrieval \cite{lewis2020retrieval, guu2020realm, ram2023context, asai2024self}, personalized and diverse recommendations \citep{gan2020enhancing, hirata2022solving, huang2024diversity, sun2024diversinews}, and multimodal understanding \cite{yu2019multimodal, luo2023semantic, zhang2024learnability, yu2023self}.
%%% limitations
Yet, despite their centrality, embedding spaces are often treated as black boxes, limiting interpretability and constraining targeted adaptation for downstream tasks.

%%% prior work
Research on understanding embedding spaces generally falls into two directions.
%%% first category
The first analyzes geometric properties to enhance both representational quality and interpretability.
%%%
In Natural Language Processing (NLP), post-processing methods have tackled issues such as anisotropy and instability \cite{mu2018all, liu2019unsupervised}, while studies on contextual embeddings reveal expressivity limits \cite{ethayarajh2019contextual}. 
Techniques like rotation-based alignment and probing further link embedding dimensions to interpretable concepts \cite{park2017rotated, dufter2019analytical, dalvi2019one, clark2019does}.
%%%
Similar geometric challenges, e.g., feature collapse and variance concentration, are found in visual \cite{chen2020simple, he2020momentum, grill2020bootstrap} and multimodal embeddings \cite{radford2021learning, jia2021scaling}.

%%% second category
The second direction explores latent semantic and hierarchical structures within embedding spaces.
%%%
In NLP, embeddings have been mapped to external conceptual systems for flexible semantic interpretations \cite{simhi2023interpreting}, while in vision, researchers have identified neurons encoding abstract, multimodal concepts \cite{goh2021multimodal}.
%%%
Unsupervised clustering \cite{van2020scan, caron2020unsupervised} uncover semantic groupings, and hierarchical classification approaches \cite{deng2012hedging, dhall2020hierarchical} organize semantics into multi-level taxonomies.

%%% NEW THREE CHALLENGES in understanding the embedding spaces
Despite substantial progress, embedding spaces remain intrinsically opaque due to three fundamental challenges:
\begin{itemize}
  \item \textbf{Abstract Semantics:} 
  Embeddings inhabit high dimensional spaces where complex, abstract relationships emerge, defying straightforward interpretation. 
  Many existing methods enhance interoperability by restructuring embeddings, but in doing so, they often distort the native geometry, thereby undermining their practical utility.
  
  \item \textbf{Lack of Explicit Structure:} 
  While semantics are inherently structured, real-world embeddings often exhibit diffuse and irregular patterns. 
  Existing methods either overlook this latent structure or impose rigid taxonomies, limiting flexibility across tasks and domains.

  \item \textbf{Limited Modality Generalization:} 
  Semantic meaning transcends text to include visual, auditory, and other modalities. 
  However, most techniques are modality-specific and lack a unified framework for revealing semantic structures across diverse embedding spaces.
\end{itemize}

%%% our contributions: 
This paper investigates semantic structures directly within native embedding spaces,
without re-embedding, restructuring, or imposing external constraints that alter their original geometry.
%%% address the first challenges
To address the abstract and opaque nature of high-dimensional embeddings, we introduce a new semantic representation, \textbf{Semantic Fields Subspaces (SFSes)}, along with \textbf{SAFARI} (\textbf{S}em\textbf{A}ntic \textbf{F}ield subsp\textbf{A}ce dete\textbf{R}m\textbf{I}nation), an unsupervised, modality-agnostic algorithm that discovers and organizes semantic structures hierarchically.
Our key contributions are as follows:
\begin{itemize}
  %%% Addressing the "Abstract Nature" challenge
  \item \textbf{Interpretable Semantic Representation:} 
  We introduce SFSes, a context-aware, geometry-preserving representation that captures semantic meaning through local neighborhoods, offering interpretability without distorting the embedding space.

  %%% Addressing the unstructured distribution challenge
  \item \textbf{Unsupervised Hierarchical Structure Discovery:} 
  We propose SAFARI, which leverages a novel \emph{Semantic Shift} metric to uncover hierarchical structures.
  A scalable approximation of Semantic Shift enables SAFARI to process large datasets with minimal accuracy loss.
  
  %%% multimodal challenge
  \item \textbf{Modality-Agnostic Generalization:} 
  SFSes and SAFARI are inherently modality-agnostic, uncovering hierarchical semantic structures across both text and image modalities without supervision or external ontologies.
\end{itemize}

%%% experiments
We validate our framework on six real-world datasets across text and image modalities.
%%% SAFARI & SFS
\textbf{SAFARI} uncovers how local neighborhoods form meaningful global hierarchies, while \textbf{SFSes} outperform standard classifiers on text classification, deliver competitive image classification with far lower computational cost, and capture subtle semantics (e.g., political bias) often missed by conventional methods.
%%% efficiency
Our Semantic Shift approximation achieves a 15$\sim$30$\times$ speedup over full SVD, with average errors below 0.01, ensuring both efficiency and accuracy.
%%%
Together, \textbf{SFSes} and \textbf{SAFARI} form a unified, interpretable, and scalable framework for semantic understanding within embedding spaces.

\section{Related Work}
\label{sect:related}

Research on understanding embedding spaces generally follows two directions: (1) structural analysis of geometric properties and (2) discovery of latent semantic hierarchies.

\paragraph{Structural Analysis of Embedding Spaces}
Early work focused on the geometric properties of embeddings, particularly in NLP.
\citet{mu2018all} mitigated anisotropy by removing dominant principal components, while \citet{liu2019unsupervised} stabilized embedding distributions by suppressing high-variance dimensions.
%%%
\citet{ethayarajh2019contextual} found that contextualized embeddings often cluster in narrow cones, limiting expressiveness.
%%%
To improve interpretability, rotation-based alignment \cite{park2017rotated, dufter2019analytical} and probing techniques \cite{dalvi2019one, clark2019does} linked embedding dimensions to human-understandable concepts.

Similar geometric issues, such as feature collapse and variance concentration, also arise in visual representations from self-supervised models like SimCLR \cite{chen2020simple}, MoCo \cite{he2020momentum}, and BYOL \cite{grill2020bootstrap}. 
\citet{wang2020understanding} analyzed such effects in contrastive learning, while feature visualization methods \cite{olah2017feature, zhou2016learning} offered neuron-level insights.
%%%
Recent advances in multimodal embeddings, e.g., CLIP \cite{radford2021learning}, ALIGN \cite{jia2021scaling}, and DeCLIP \cite{li2022supervision}, explore joint semantic alignment across modalities.
%%%
Unlike these approaches, which often modify embedding spaces, \textbf{SAFARI} preserves native geometry while enhancing interpretability.

\paragraph{Semantic and Hierarchical Structure Discovery}
A complementary line of research seeks to uncover semantic groupings and hierarchies within embedding spaces.
In NLP, \citet{simhi2023interpreting} maps embeddings into conceptual spaces grounded in knowledge bases, enabling flexible interpretations.
In vision, \citet{goh2021multimodal} found that certain CLIP neurons respond to abstract concepts shared across modalities. 
%%%
Unsupervised methods \cite{van2020scan, caron2020unsupervised} discover coherent visual clusters but lack hierarchical organization.
Conversely, hierarchical classification methods \cite{deng2012hedging, dhall2020hierarchical}  build multi-level structures using taxonomies like WordNet, yet depend on external supervision and predefined label trees.
%%%
\textbf{SAFARI} bridges these gaps with a unified, unsupervised, and modality-agnostic framework that identifies semantic hierarchies directly from embedding spaces--without structural transformation or supervision.

\section{Problem Formulation}
\label{sect:problem}

\paragraph{Vector Space Foundation}
We model embedding spaces as a vector space $\mathbb{R}^d$. 
For ease of explaining the core concepts in this paper, we use natural language terms as illustrative examples, though all concepts introduced are modality-agnostic.
%%%
Let $h: \Set{T} \rightarrow \Set{E}$ be a deep model that maps real-world terms $\Set{T}$ to embedding vectors $\Set{E}$.
A central assumption is that \textbf{geometric distances reflect semantic similarities}, a principle validated in various tasks ~\cite{karpukhin2020dense, lewis2020retrieval, ram2023context}.
%%%
In this work, we adopt cosine distance as a proxy for semantic dissimilarity. 
\begin{definition}[Semantic Distance]
\label{def:semantic_distance}
  The semantic distance~$d_{sem}(\cdot,\cdot)$ between two embedding vectors $\Vec{u}, \Vec{v} \in \Space{R}^d$ is defined as $d_{sem}(\Vec{u}, \Vec{v}) = 1 - \tfrac{\ip{\Vec{u}, \Vec{v}}}{\Norm{\Vec{u}} \Norm{\Vec{v}}}$.
\end{definition}

\paragraph{Challenge of Context-Dependent Meaning}
Despite precise embeddings, \textbf{semantic meaning depends on context}. 
Linguistic theories such as semantic field theory and componential analysis~\cite{ullmann1957principles, nida2015componential} argue that meaning arises only through contextual associations.
\begin{proposition}[Context-Dependent Meaning]
\label{prop:word_context}
  An embedding vector cannot be semantically interpreted in isolation.
\end{proposition}

\begin{figure}[ht]
  \centering
  \includegraphics[width=0.99\columnwidth]{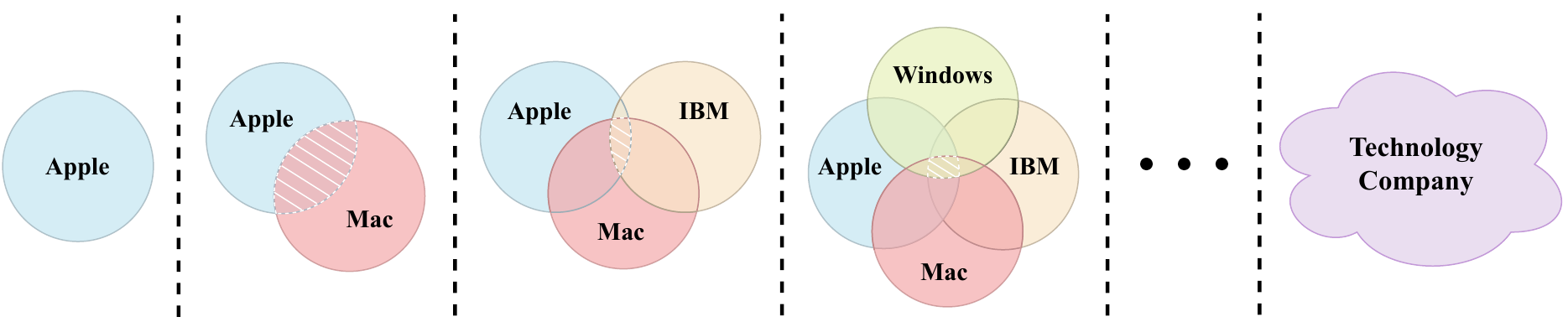}
  \caption{Contextual interpretation of \term{Apple}: Meaning refines as more related terms are introduced.}
  \label{fig:apple_example}
\end{figure}

\begin{example}
\label{exp:apple}
  The term \term{Apple} is semantically ambiguous and relies on context for disambiguation.
  As shown in Figure \ref{fig:apple_example}, it refers to a tech company when grouped with \term{Mac}, \term{IBM}, and \term{Windows}, \term{Apple}, but to a fruit with \term{Apple Tree}, \term{Juice}, and \term{Banana}. 
  More contextual cues yield more precise interpretations, highlighting the context-dependent nature of semantics in embedding spaces.
  \hfill $\triangle$ \par 
\end{example}

\paragraph{Semantic Fields in Embedding Spaces}
To model context-dependent semantics, we define \textbf{Semantic Fields} as sets of neighboring vectors that contextualize a target embedding vector.
For instance, as shown in Example~\ref{exp:apple}, the meaning of \term{Apple} becomes clearer when surrounded by neighbors like \term{Mac}, \term{IBM}, and \term{Windows}, forming a Semantic Field of \term{Apple}.
%%%
This concept aligns with foundational embedding models such as Word2Vec~\cite{Mikolov2013Word2Vec} and BERT~\cite{devlin2019bert}, where semantics arise from context.
%%%
To formalize this, we distinguish general embedding vectors ($\Vec{v}$) from those representing real-world terms ($\Vec{v}_t$), where the latter excludes purely mathematical constructs (e.g., zero vectors).
\begin{definition}[Semantic Field]
\label{def:semantic_field}
  A set $\Set{F}$ of embedding vectors forms a Semantic Field of radius $\epsilon > 0$ if there exists a central vector $\Vec{v}_t \in \Set{F}$ such that for all $\Vec{u}_t \in \Set{F}$, we have $d_{sem}(\Vec{u}_t, \Vec{v}_t) < \epsilon$.
\end{definition}
While Definition~\ref{def:semantic_field} allows us to examine local structures in an embedding space, it is insufficient for understanding the global organization of semantics.

\paragraph{Research Objective: From Local to Global Semantics}
Our goal is to investigate how Semantic Fields collectively shape the global semantic structure of embedding spaces, uncovering how these local Semantic Fields relate, interact, and form coherent semantic hierarchies.
%%%
To this end, we propose \textbf{SAFARI}, a principled framework that detects and analyzes hierarchical semantic structures by identifying boundaries between Semantic Fields. 
This method offers an interpretable len for understanding how semantics are organized in high-dimensional embedding spaces.

\section{Methodology}
\label{sect:method}

\subsection{Semantic Field Representation}
\label{sect:method:motivation}

By Definition~\ref{def:semantic_field}, a Semantic Field is the neighborhood of a target embedding vector.
\textbf{SAFARI} identifies the structure of these neighborhoods.
However, representing such structures is non-trivial due to the following two key challenges:
\begin{itemize}
  \item \textbf{Handling Expression Variants:} 
  Closest neighbors are often populated by variants that are overly similar to the original, offering limited values into its interpretation~\cite{mimno2017strange, ethayarajh2019contextual}. 
  For example, word synonyms or images of the same object under different lighting conditions contribute little to understanding the underlying concept. 
  A robust representation should exhibit invariance to such variations.
  
  \item \textbf{Delineating Semantic Field Boundaries:}
  As a Semantic Field expands, it includes more embedding vectors in larger neighborhoods and thus provides richer context for more refined interpretations of the target embedding.
  However, since semantics naturally form hierarchies, there exists a boundary beyond which the Semantic Field contains embeddings diverse enough for it to transcend from a concrete concept to a more abstract one.
  It is necessary to determine such boundaries to tell when the Semantic Field starts to represent a new concept.
\end{itemize}

As illustrated in Figure \ref{fig:close_neighborhood}, starting from \term{Coca-Cola}, the nearest neighbor is \term{Coke} (a lexical variant), while broader semantic context only occurs with terms like \term{Sprite} and \term{Pepsi}, underscoring the difficulty of managing expression variants and identifying natural semantic boundaries.

\begin{figure}[ht]
  \centering
  \includegraphics[width=0.6\columnwidth]{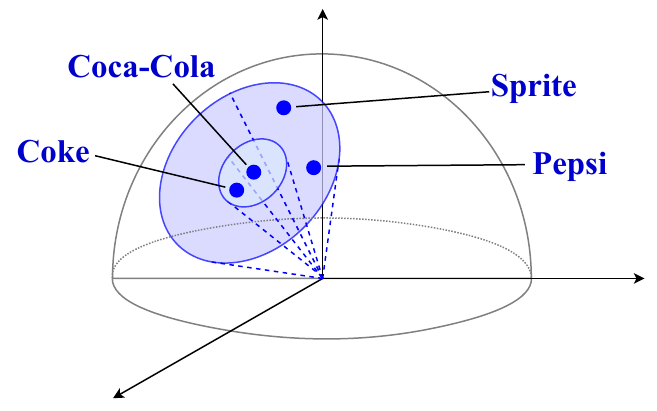}
  \caption{Illustration of Semantic Field exploration.}
  \label{fig:close_neighborhood}
\end{figure}

\paragraph{Geometric Representation: Semantic Field Subspace (SFS)} 
To address these challenges, we introduce the Semantic Field Subspaces (SFS), a low-dimensional subspace spanned by semantically related vectors. 
This geometric representation naturally absorbs expression variants via linear dependence and provides a compact, geometry-preserving abstraction of semantic content.

\begin{definition}[Semantic Field Subspace (SFS)]
\label{def:semantic_field_subspace}
  Let $\Set{F} = \{\Vec{v}_{1}, \cdots, \Vec{v}_{n}\}$ be a Semantic Field. Its SFS is defined as: $\Space{S}_{\Set{F}} = \text{span}(\Set{F}) = \{ \textstyle \sum_{i=1}^n \alpha_i \Vec{v}_{i} \mid \alpha_i \in \Space{R} \}$.
\end{definition}

We compute the basis of $\Space{S}_{\Set{F}}$ via SVD on matrix $\Mat{M}_{\Set{F}} = [\Vec{v}_1, \cdots, \Vec{v}_n]$, i.e., $\Mat{M}_{\Set{F}} = \Mat{U} \Mat{\Sigma} \Mat{V}^{\top}$.
This yields a continuous, low-rank representation that captures semantic structure while remaining invariant to redundancy.

\begin{figure*}[t]
  \centering
  \includegraphics[width=0.99\textwidth]{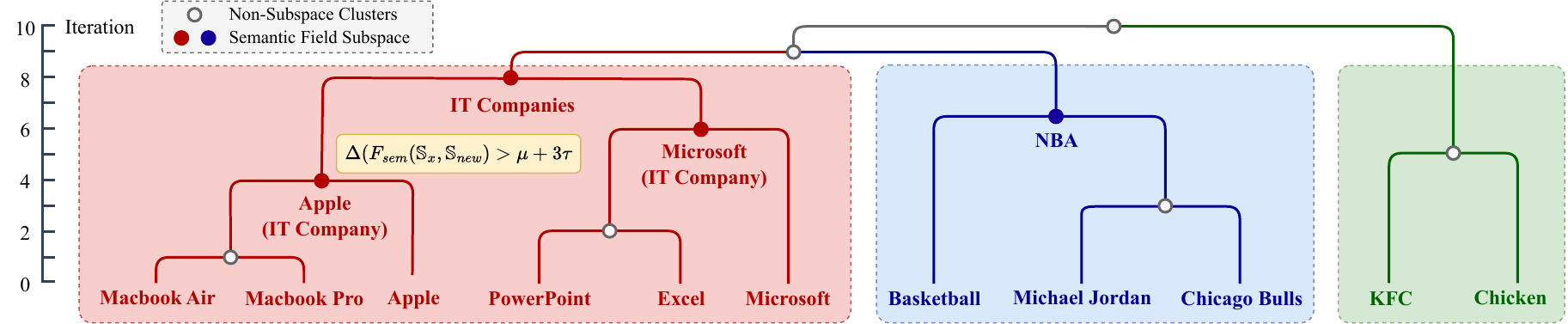}
  \caption{Toy example illustrating \textbf{SAFARI}'s hierarchical clustering process.}
  \label{fig:safari_example}
\end{figure*}

\paragraph{SFS Boundary Delineation via Semantic Shift}
To delineate boundaries between Semantic Fields, we leverage the hierarchical nature of language semantics, where expanding fields lead to broader, more abstract concepts.
\begin{proposition}[Hierarchical Semantic Structure]
\label{prop:hierarchy}
  Semantic hierarchies in natural language are reflected in the geometric structure of embedding spaces.
\end{proposition}

As a Semantic Field expands, a significant shift in meaning often indicates an evolution to a new field. 
We quantify this evolution via the \textbf{Semantic Shift} between two subspaces $\Space{S}_{\Set{F}_x}$ and $\Space{S}_{\Set{F}_{new}}$:
\begin{definition}[Semantic Shift]
\label{def:semantic_shift}
  The Semantic Shift between $\Space{S}_{\Set{F}_x}$ and $\Space{S}_{\Set{F}_{new}}$ is defined as:
  \begin{equation} \label{eqn:semantic_shift}
    \Delta F_{sem}(\Space{S}_{\Set{F}_x}, \Space{S}_{\Set{F}_{new}}) = \textstyle \sum_{i} \Delta \sigma_i \cdot d_{sem}(\Vec{v}_i, \tilde{\Vec{v}}_i^*),
  \end{equation}
  where $\Delta \sigma_i = |\sigma_i - \tilde{\sigma}_i|$ captures the \textbf{dimensional importance shift} in singular values $\sigma_i \in \Mat{\Sigma}_x$ and $\tilde{\sigma}_i \in \Mat{\Sigma}_{new}$; 
  $d_{sem}(\Vec{v}_i, \tilde{\Vec{v}}_i^*)$ captures \textbf{directional change} between basis vectors $\Vec{v}_i \in \Mat{V}^{\top}_x$ and their nearest counterparts $\tilde{\Vec{v}}_i^* \in \Mat{V}^{\top}_{new}$.
\end{definition}

Semantic Shift acts as a boundary criterion:
A large shift suggests that the new subspace $\Space{S}_{\Set{F}_{new}}$ represents a more abstract concept that subsumes $\Space{S}_{\Set{F}_x}$, whereas small values reflect refinements within the same semantic field.

\subsection{The SAFARI Algorithm}
\label{sect:method:safari}

\paragraph{Algorithm Overview}
Building on Definitions \ref{def:semantic_field_subspace} and \ref{def:semantic_shift}, we propose \textbf{SAFARI}, an algorithm to uncover SFSes by monitoring Semantic Shifts during iterative clustering.
%%%
At each step, \textbf{SAFARI} merges the nearest clusters, resulting in a new subspace. 
The algorithm then evaluates the Semantic Shift between the new subspace and the previous subspace and checks whether such a shift is significant. 
If so, the new subspace is identified as a new SFS that subsumes the previous subspaces.

\paragraph{Detailed Procedure}
The pseudo-code is provided in Algorithm~\ref{alg:safari}. 
%%% Initialization
\textbf{SAFARI} initializes by assigning each vector as a singleton cluster in a set $\Omega$, and maintains a set $\Phi$ to store the discovered SFSes.
%%% Iteration
It proceeds iteratively with the steps below until only one cluster remains ($|\Omega| \leq 1$):
\begin{itemize}%[nolistsep,left=0pt]
  \item \textbf{Step 1: Cluster Merging.}
  The two nearest clusters $\Set{C}_x$ and $\Set{C}_y$ are identified using Semantic Distance $d_{sem}(\Set{C}_x,\Set{C}_y)$, with centroids representing each cluster.
  They are then merged into a new cluster $\Set{C}_{new}$, after which $\Set{C}_x$ and $\Set{C}_y$ are removed, and $\Set{C}_{new}$ is added to $\Omega$.

  \item \textbf{Step 2: SFS Delineation.}
  \textbf{SAFARI} constructs the subspaces $\Space{S}_{new}$ and $\Space{S}_x$ for $\Set{C}_{new}$ and the larger cluster $\Set{C}_x$, and computes Semantic Shift $\Delta F_{sem}(\Space{S}_x, \Space{S}_{new})$. % to evaluate semantic change. 
  A sliding window of size $w$ tracks the recent $w$ values, computing mean $\mu$ and standard deviation $\tau$. 
  If the current $\Delta F_{sem}(\Space{S}_x, \Space{S}_{new})$ exceeds the dynamic threshold ($\mu + 3\tau$), $\Space{S}_{new}$ is added to $\Phi$ as a new SFS.
\end{itemize}

\begin{algorithm}[t]
\small
\caption{\textbf{SAFARI}}
\label{alg:safari}
\KwIn{Embedding set $\Set{E} \subset \Space{R}^d$, window size $w$;}
\KwOut{Set $\Phi$ of Semantic Field Subspaces (SFSes);}
$\Omega \gets$~Initialize each $\Vec{v}_t \in \Set{E}$ as its own cluster\; 
$\mu \gets 0$; $\tau \gets 0$; $\Phi \gets \varnothing$\;
\While{$\Size{\Omega} > 1$}{
  %%% Cluster Merging
  \Comment{Step 1:~Cluster Merging}
  $\{\Set{C}_x, \Set{C}_y\} \gets \argmin_{\Set{C}_i, \Set{C}_j \in \Omega} d_{sem}(\Set{C}_i, \Set{C}_j)$\;
  $\Set{C}_{new} \gets \Set{C}_x \cup \Set{C}_y$\;
  $\Omega \gets \Omega \cup \{\Set{C}_{new}\} \setminus \{\Set{C}_x, \Set{C}_y\}$\;
  
  %%% Semantic Field Subspace Determination
  \BlankLine
  \Comment{Step 2:~SFS Delineation}
  $\Set{C}_x \gets \Size{\Set{C}_x} > \Size{\Set{C}_y}~?~\Set{C}_x~:~\Set{C}_y$\;
  $\Space{S}_{x}, \Space{S}_{new} \gets \text{span}(\Set{C}_{x}), \text{span}(\Set{C}_{new})$\;
  Compute $\Delta F_{sem}(\Space{S}_{x}, \Space{S}_{new})$ using Eq.~\ref{eqn:semantic_shift}\;
  \If{$\Delta F_{sem}(\Space{S}_{x}, \Space{S}_{new}) > \mu + 3 \tau$}{
    $\Phi \gets \Phi \cup \{\Space{S}_{new}\}$\; 
  }
  %%% maintain extra info
  Update $\mu$ and $\tau$ using the last $w$ values of $\Delta F_{sem}$\; \label{alg:safari:update-params}
} \label{alg:safari:iter_end}
\Return $\Phi$\;
\end{algorithm}
\setlength{\textfloatsep}{1.0em}

\paragraph{Remarks on Design}
The use of a sliding window ($w$) and the dynamic threshold ($\mu + 3\tau$), which is obtained via parameter study, is essential to accurately identify SFSes because we empirically observe that the baseline values of Semantic Shifts grow gradually as the algorithm proceeds.
The dynamic threshold allows \textbf{SAFARI} to adapt to such gradual change and effectively detect local spikes of real significance. 
%%%
% Moreover, \textbf{SAFARI} benefits from \textbf{hierarchical clustering} in two key ways: 
Moreover, \textbf{SAFARI} adopts the hierarchical process to reflect the layered nature of semantic relationships:
(1) It dynamically determines SFSes based on semantic structures, avoiding relying on pre-defined cluster counts; 
(2) It reveals natural semantic hierarchies through the resulting dendrogram.

\begin{example}
\label{exp:safari}
Consider a toy dataset of 11 textual terms. 
Figure \ref{fig:safari_example} shows how \textbf{SAFARI} identifies SFSes.
%%%
In the first three iterations, semantically close pairs (e.g., \term{Macbook Air} \& \term{Macbook Pro}, \term{PowerPoint} \& \term{Excel}, and \term{Michael Jordan} \& \term{Chicago Bulls}) are merged without forming SFSes.
%%%
In the 4th iteration, merging \term{Apple} with the \term{Macbook} cluster triggers a significant Semantic Shift, forming a new SFS.
%%%
By the 8th iteration, a hierarchical structure emerges with an \term{IT Companies} subspace encompassing nested SFSes for \term{Apple (IT Company)} and \term{Microsoft (IT Company)}.
%%%
In the 9th iteration, the dynamic threshold prevents the merge between \term{IT Companies} and \term{NBA}, preserving semantic boundaries.
\hfill $\triangle$ \par
\end{example}

\subsection{Efficient Approximation of Semantic Shift}
\label{sect:method:approximation}

Each iteration of \textbf{SAFARI} requires computing the Semantic Shift via full SVD on matrices of size $n \times d$ ($d \leq n$), incurring a time complexity of $O(nd^2)$ \cite{halko2009finding, trefethen2022numerical}. 
This becomes a major bottleneck in large-scale applications.
%%%
To improve scalability, we propose a practical approximation. 
Let $\Mat{A}_x$ and $\Mat{A}_y$ be the matrices of a large cluster $\Set{C}_x$ and a small one $\Set{C}_y$. 
Instead of computing full SVDs, we approximate the Semantic Shift between $\Space{S}_{x}$ and $\Space{S}_{new}$ as:
\begin{equation} \label{eqn:appro_semantic_shift}
  \Delta \tilde{F}_{sem}(\Space{S}_{x}, \Space{S}_{new}) = \Norm{\Mat{A}_{y}}_2 \sigma_{max}(\Mat{A}_x),
\end{equation}
where $\Norm{\Mat{A}_{y}}_2$ is the spectral norm of $\Mat{A}_y$, and $\sigma_{max}(\Mat{A}_x)$ is the largest singular value of $\Mat{A}_x$.
This yields substantial speedups with negligible loss of accuracy (see Section \ref{sect:expt:efficiency}).

\paragraph{Theoretical Justification}
Let $\Mat{A}_{new} = [\Mat{A}_x | \Mat{A}_y]$ be the matrix representing the newly merged cluster $\Set{C}_{new}$. 
We now justify the approximation by establishing two key theoretical results:
\begin{itemize}
  \item An upper bound on the dimensional importance shift in singular values;
  \item A connection between directional change and the largest singular value of $\Mat{A}_{x}$.
\end{itemize}

\paragraph{Bounding Dimensional Importance Shift}
We begin with the following result:
\begin{theorem}[Bound on Dimensional Importance Shift]
\label{thm:aug_form}
  Given matrices $\Mat{A}_x$ and $\Mat{A}_y$ with the same number of columns and assuming $\Mat{A}_x$ has more rows than $\Mat{A}_y$, the shift in the $i$-th singular value satisfies:
  \begin{displaymath}
    \Delta \sigma_i = |\sigma_i(\Mat{A}_x)-\sigma_i(\Mat{A}_{new})| \leq \Norm{\Mat{A}_y}_2.
  \end{displaymath}
\end{theorem}

\begin{proof}
The result follows from Weyl's Theorem \cite{weyl1912asymptotische}, which bounds the change in singular values under additive perturbations.
Consider the larger cluster represented by matrix $\Mat{A} \in \Space{R}^{m \times d}$.
When merging with another cluster, the resulting matrix $\Mat{\tilde{A}}$ can be viewed as a perturbed version of $\Mat{A}$, i.e., $\Mat{\tilde{A}} = \Mat{A} + \Mat{E}$, where $\Mat{E}$ is the perturbed matrix. 
\begin{theorem}[Weyl's Theorem \cite{weyl1912asymptotische}]
\label{thm:weyls}
  For any perturbed matrix $\Mat{E}$, the singular values satisfy: $|\sigma_i(\Mat{A})-\sigma_i(\Mat{\tilde{A}})| = |\sigma_i(\Mat{A})-\sigma_i(\Mat{A} + \Mat{E})| \leq \Norm{\bm{E}}_2$.
\end{theorem}

This result implies that the change in any singular value is at most the spectral norm of the perturbed matrix, regardless of its dimension \cite{stewart1998perturbation}.
%%%
To apply this, we rewrite $\Mat{A}_{new} = [\Mat{A}_x | \Mat{A}_y] = [\Mat{A}_x|\Mat{O}] + [\Mat{O}|\Mat{A}_y]$, where $\Mat{O}$ is a zero matrix. 
According to Theorem \ref{thm:weyls}, we have:
\begin{align*}
  |\sigma_i([\Mat{A}_x | \Mat{O}]) - \sigma_i(\Mat{A}_{new})| 
  &= |\sigma_i([\Mat{A}_x | \Mat{O}]) - \sigma_i([\Mat{A}_x | \Mat{A}_y])| \\
  & \leq \Norm{[\Mat{O}|\Mat{A}_y]}_2 = \Norm{\Mat{A}_y}_2.
\end{align*}
Since $\sigma_i([\Mat{A}_x | \Mat{O}]) = \sigma_i(\Mat{A}_x)$, Theorem \ref{thm:aug_form} is proved.
\end{proof}

\paragraph{Approximating Directional Change}
Intuitively, $\sum_{i} d_{sem}(\Vec{v}_i, \tVec{v}_i^*)$ captures the perturbation in basis directions during cluster merging.
Following \cite{meyer2000matrix, belsley2005regression}, it is expected that a higher sensitivity in $\Mat{A}_x$, typically characterized by its condition number $\kappa(\Mat{A}_x) = {\sigma_{max}(\Mat{A}_x)}/{\sigma_{min}(\Mat{A}_x)}$, will result in larger directional changes when $\Mat{A}_x$ and $\Mat{A}_y$ are merged, i.e., $\sum_{i} d_{sem}(\Vec{v}_i, \tVec{v}_i^*) = \mathcal{O}(\kappa(\Mat{A}_x))$.

Since $\Mat{A}_x$ and $\Mat{A}_y$ represent embedding vectors sampled similarly from the embedding space, the impact of noise on these vectors are similar. 
Therefore, their minimum singular values can be assumed comparable, i.e, $\sigma_{min}(\Mat{A}_x) \approx \sigma_{min}(\Mat{A}_y) \approx \sigma_{min}(\Mat{A}_{new})$.
This allows us to discard the effects of minimum singular values, and the directional change can be approximated by $\sigma_{max}(\Mat{A}_x)$, i.e., $\sigma_{max}(\Mat{A}_x) \approx \mathcal{O}(\kappa(\Mat{A}_x)) \approx \sum_{i} d_{sem}(\Vec{v}_i, \tVec{v}_i^*)$.

While this approximation is necessarily coarse, it offers an efficient alternative for real-time clustering. 
Empirical results (Section \ref{sect:expt:case_study}) confirm its efficacy with minimal impact on performance.

\section{Experiments}
\label{sect:expt}

\subsection{Datasets and Experiment Environment}
\label{sect:expt:dataset}

\paragraph{Datasets}
We evaluate \textbf{SAFARI} and \textbf{SFSes} on six public datasets across text and image modalities.
%%%
For the text modality, we use five diverse datasets: \textbf{AG-News} \cite{zhang2015character} (with 4 topic classes: Business, Sci/Tech, Sports, and World), \textbf{AAPD} \cite{yang2018sgm}, \textbf{IMDB} \cite{maas2011learning}, \textbf{Yelp},\footnote{\url{https://www.yelp.com/dataset}} and \textbf{NewsSpectrum} \cite{sun2024diversinews}.
%%%
For the image modality, we employ \textbf{MIT-States} \cite{isola2015discovering}, which enables the evaluation of object-attribute composition in visual embeddings.
%%%
Further dataset details are available in Appendix A.1.

\paragraph{Experiment Environment}
All methods were implemented in Python 3.8 and evaluated on a Ubuntu 20.04 machine with Intel\textsuperscript{\textregistered} Xeon\textsuperscript{\textregistered} Platinum 8480C and an NVIDIA H100 GPU.

\subsection{Hierarchical Semantic Structure Discovery} % in Embedding Spaces
\label{sect:expt:hierarchical_structure}

\paragraph{Experimental Setup} 
%%% Dataset
We evaluate \textbf{SAFARI} and the resulting \textbf{SFSes} on two modality-diverse datasets: \textbf{AG-News} for text and \textbf{MIT-States} for images.
%%% Model
Embedding spaces are constructed using BLINK \cite{wu2020scalable} for AG-News and CLIP \cite{radford2021learning} for MIT-States.
%%% Hierarchy
To benchmark hierarchical discovery, we generate 4-level semantic label hierarchies (Lv0 to Lv3) for both datasets using Claude 3.7 Sonnet \cite{anthropic2025claude}, with details provided in Appendix A.2.
%%% Metric
We assess how well the \textbf{SFSes} align with the reference hierarchy using the \textbf{impurity} metric:
\begin{displaymath}
  \text{Impurity} = \textstyle \frac{1}{l} \sum_{i=1}^l (1-\frac{1}{|L_i|} \max_{1 \leq j \leq c} |L_i \cap C_j|),
  % \text{Impurity} = \frac{1}{l} \sum_{i=1}^l \big(1-\frac{1}{|L_i|} \max_{1 \leq j \leq c} |L_i \cap C_j| \big),
\end{displaymath}
where $L_i$ is label class $i$; $C_j$ is cluster $j$ used to construct \textbf{SFSes}; $l$ and $c$ are the number of label classes and clusters. 
Lower impurity indicates greater semantic coherence within clusters, with 0 denoting perfect label concentration within clusters.
%%% How to interpret the trend
As \textbf{SAFARI} progresses, it is expected that impurity grows due to the merging of semantically broader categories, with a consistent ordering across levels: Lv0 (most specific) $>$ Lv1 $>$ Lv2 $>$ Lv3 (most abstract).

\begin{figure}[ht]%
\centering%
\includegraphics[width=0.6\columnwidth]{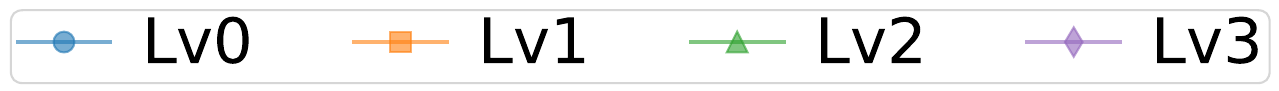} \\%
\subfigure[AG-News dataset.]{%
  \label{fig:hierarchy:agnews-diversity}%
  \includegraphics[width=0.48\columnwidth]{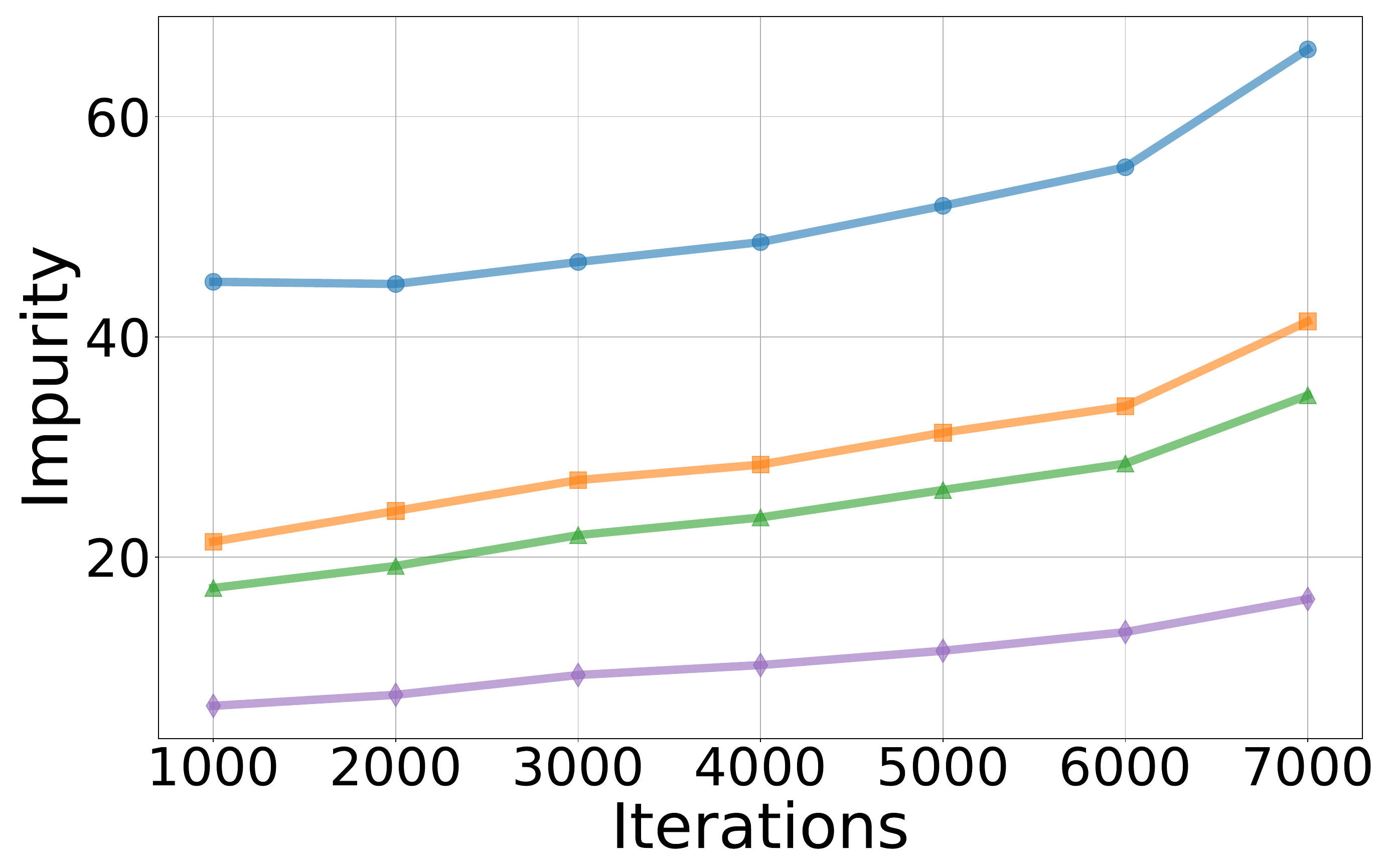}}%
\hfill%
\subfigure[MIT-States dataset.]{%
  \label{fig:hierarchy:mit-diversity}%
  \includegraphics[width=0.48\columnwidth]{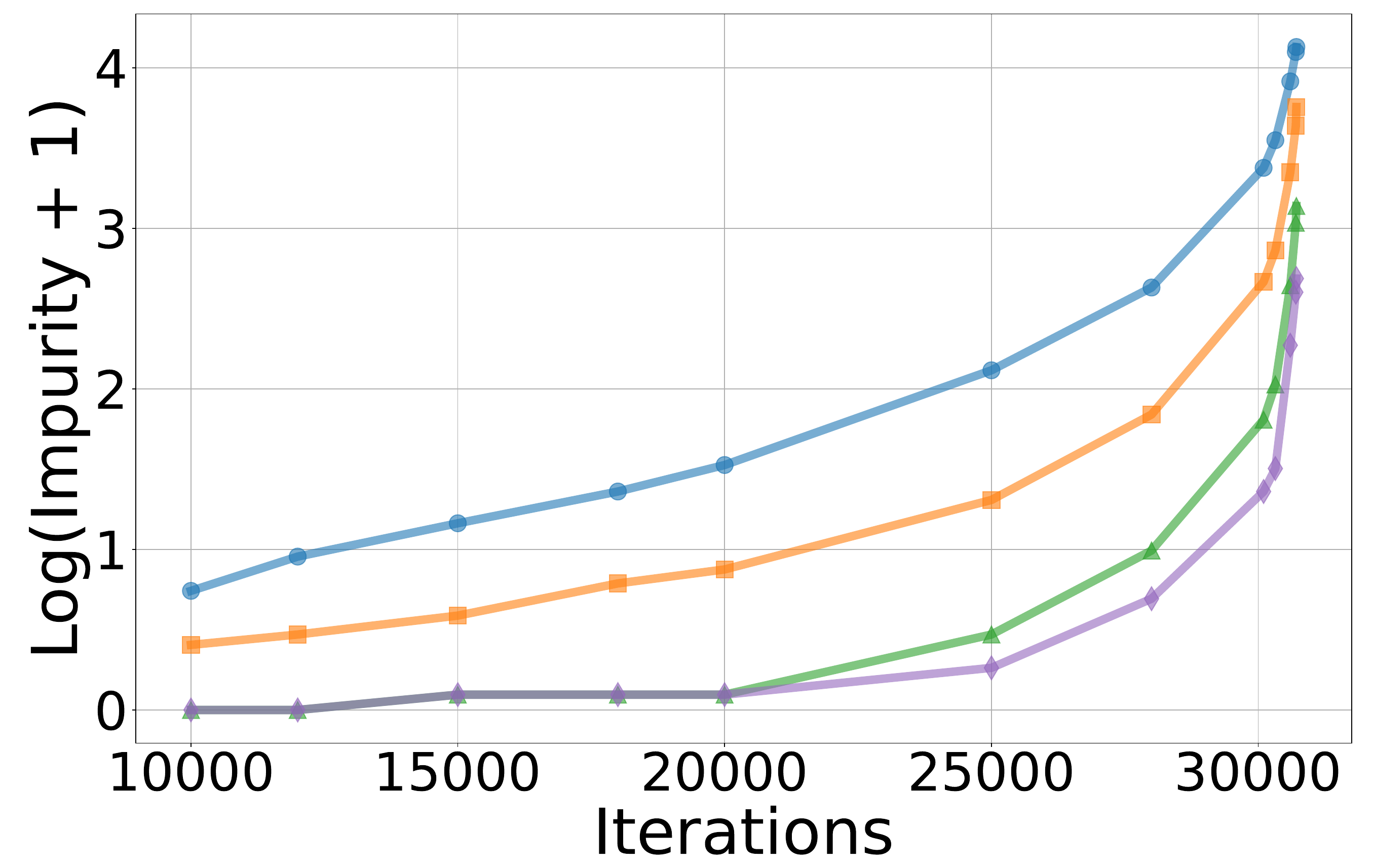}}%
\caption{Impurity across hierarchical levels.}%
\label{fig:hierarchy-diversification}%
\end{figure}%

\begin{figure*}[t]
  \centering
  \includegraphics[width=0.99\textwidth]{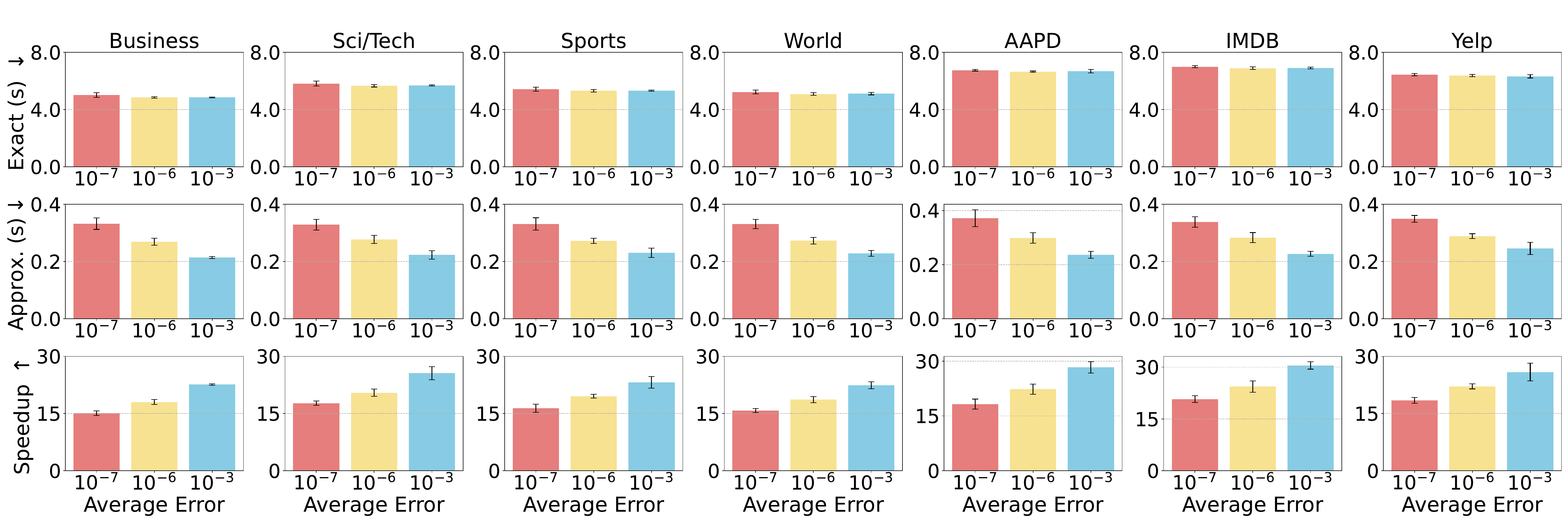}
  \caption{Runtime comparison between exact and approximate Semantic Shift computation across seven topic classes.}
  \label{fig:efficiency}
\end{figure*}

\begin{table*}[t]
\centering
\small
\resizebox{0.99\textwidth}{!}{%
\begin{tabular}{c cccc cccc}
  \toprule
  \multirow{2.5}{*}{\textbf{Method}} & \multicolumn{4}{c}{\textbf{Text Classification}} & \multicolumn{4}{c}{\textbf{Image Classification}} \\ \cmidrule(lr){2-5} \cmidrule(lr){6-9} 
  & \textbf{Precision (\%) $\uparrow$} & \textbf{Recall (\%) $\uparrow$} & \textbf{F1-score (\%) $\uparrow$} & \textbf{Time (s) $\downarrow$} & \textbf{Precision (\%) $\uparrow$} & \textbf{Recall (\%) $\uparrow$} & \textbf{F1-score (\%) $\uparrow$} & \textbf{Time (s) $\downarrow$} \\ 
  \midrule
  \textbf{SAFARI}                  
  & \textbf{48.3} & \textbf{49.3} & \textbf{48.5} & 46.37  & \underline{61.4} & \underline{61.1} & \underline{60.5} & \underline{18.64}    \\
  \textbf{SVM}                     
  & \underline{47.5} & \underline{47.8} & \underline{47.6} & 91.87  & \textbf{63.5} & \textbf{62.7} & \textbf{62.1} & 69.58    \\
  \textbf{KNN}                     
  & 41.9 & 43.1 & 42.1 & \textbf{1.167}   & 58.5 & 57.5 & 57.0 & \textbf{2.855}     \\
  \textbf{MLP}                     
  & 43.4 & 41.6 & 42.1 & 111.5   & 54.7 & 54.4 & 54.3 & 92.55  \\
  \textbf{RF}                      
  & 35.7 & 40.3 & 37.8  & \underline{35.64}  & 56.0 & 55.3 & 54.0 & 105.6  \\ 
  \bottomrule
\end{tabular}}
\caption{Classification results for text and image modalities. \textbf{Bold} and \underline{underlined} denote the best and second-best scores, respectively. \textbf{SAFARI} leads on text classification and ranks second on image classification, balancing accuracy and efficiency.}
\label{tab:classification_results}
\end{table*}
\setlength{\textfloatsep}{1.0em}

\paragraph{Results and Analysis}  
Figure \ref{fig:hierarchy-diversification} shows a consistent decrease in impurity from Lv0 to Lv3 across iterations for both modalities.
%%%
This trend reflects a hierarchical shift from specific to abstract semantics, confirming that \textbf{SFSes} capture coherent semantic groupings at multiple granularities.

%%% Highlight the modality-agnostic nature
The consistent patterns across text and image confirm the \textbf{modality-agnostic} nature of \textbf{SAFARI}.
%%% 
Without supervision, it uncovers increasingly abstract semantic relationships by tracking Semantic Shifts.
%%% Confirm the eventual goal of this paper
By revealing how local neighborhoods compose global hierarchies, \textbf{SAFARI} offers a robust and generalizable framework for identifying hierarchical structures in embedding spaces, advancing our understanding of their inherent semantic organization.

\subsection{Classification Across Modalities}
\label{sect:expt:classification}

\paragraph{Experimental Setup}
We measure whether \textbf{SFSes} preserve meaningful semantics by testing their performance on classification tasks across text and image modalities.
%%% datasets
For text classification, we use four datasets: \textbf{AG-News} (4 topics), \textbf{AAPD}, \textbf{IMDB}, and \textbf{Yelp}, covering seven distinct classes.
For image classification, we use \textbf{MIT-States}, filtering to 97 object classes with at least 240 samples each.
%%% baselines
We compare \textbf{SAFARI} against four standard classifiers: Support Vector Machine (SVM) \cite{platt1999probabilistic, chang2011libsvm}, K-Nearest Neighbors (KNN) \cite{cover1967nearest, fix1985discriminatory}, Random Forest (RF) \cite{breiman2001random}, and Multi-Layer Perceptron (MLP) \cite{he2015delving, hinton1990connectionist}. 
For classification using \textbf{SFSes}, we compute the distance between each test embedding and all identified \textbf{SFSes}, assigning the label to the nearest one.
%%% results
Results are presented in Table~\ref{tab:classification_results}.

\paragraph{Results and Analysis}
For text classification, \textbf{SAFARI} outperforms all baselines, surpassing SVM, the second-best, while using only around 50\% of the computation time.
In image classification, it ranks second to SVM in accuracy but runs 3.7$\times$ faster, demonstrating a strong accuracy-efficiency trade-off.
%%%
Overall, \textbf{SAFARI} delivers competitive performance across modalities with notable computational savings, confirming that \textbf{SFSes} offer effective and efficient semantic representations.

We further evaluate political bias detection on the \textbf{NewsSpectrum} dataset (details in Appendix B).
%%%
Results show that \textbf{SAFARI} successfully captures nuanced ideological distinctions, where standard classifiers often fail, particularly on underrepresented political leanings, highlighting its robustness in modeling subtle, real-world semantics beyond surface-level topics.

\subsection{Efficient Semantic Shift Approximation}
\label{sect:expt:efficiency}

\begin{figure*}[t]
  \centering
  \includegraphics[width=0.99\textwidth]{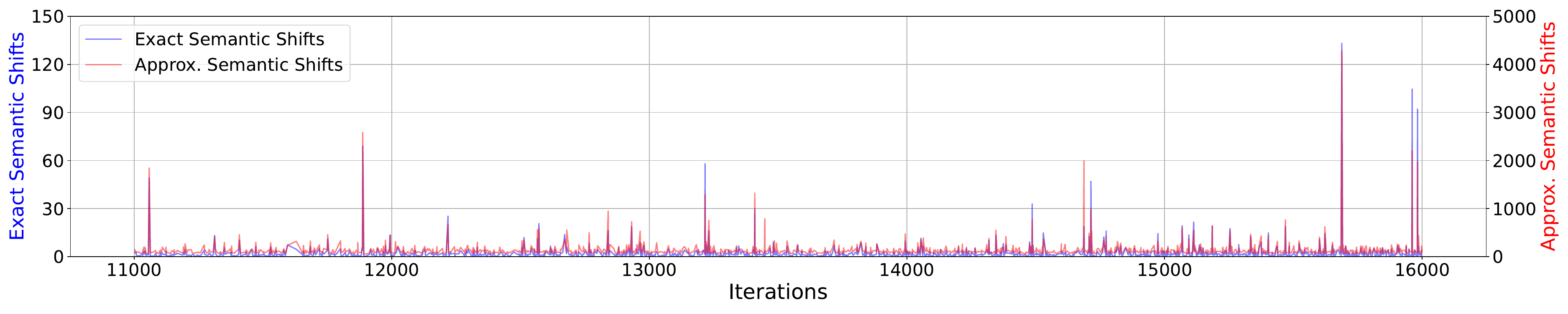}
  \caption{Semantic Shift (exact vs. approximate) on the Sports category in AG-News.} 
  \label{fig:exp2_iter}
\end{figure*}

\paragraph{Experimental Setup}
To evaluate the efficiency of \textbf{SAFARI}'s approximate Semantic Shift computation (Equation \ref{eqn:appro_semantic_shift}), we compare it against the full SVD-based method (Equation \ref{eqn:semantic_shift}), in terms of runtime. 
%%%
Following the setup in Section \ref{sect:expt:classification}, we sample the top 2,000 entities from each dataset and perform hierarchical clustering. Figure \ref{fig:efficiency} depicts the results, with runtime averaged over 10 independent runs.

\paragraph{Results and Analysis}
As shown in Figure \ref{fig:efficiency}, the approximate method achieves a 15$\sim$30$\times$ speedup over full SVD across all classes, with consistently low variance (as indicated by the error bars).
%%%
Despite the substantial acceleration, the average error between exact and approximate Semantic Shifts remains below 0.01, within the $10^{-3}$ scale, ensuring strong accuracy-efficiency trade-offs.
%%%
These results confirm that our approximation is a fast, stable, and reliable alternative to full SVD, making \textbf{SAFARI} scalable for large datasets.
%%% DIC and DC
A detailed analysis of the two Semantic Shift components, i.e., Dimensional Importance Shift and Directional Change, is provided in Appendix C.

\subsection{Case Study: Discovering Semantic Hierarchies in AG-News}
\label{sect:expt:case_study}

\paragraph{Experimental Setup}
%%% dataset
To illustrate \textbf{SAFARI}'s ability to uncover hierarchical semantics, we conduct a case study on the \textbf{Sports} category of the \textbf{AG-News} dataset, chosen for its structured, event-driven content.
%%% method
We apply \textbf{SAFARI} to the top 2,000 entity embeddings, computing Semantic Shifts at each iteration using both the exact (Equation \ref{eqn:semantic_shift}) and approximate (Equation \ref{eqn:appro_semantic_shift}) methods.
%%% how to visualize
Figure \ref{fig:exp2_iter} plots Semantic Shift curves between iterations 11,000 and 16,000, with notable evolutions at iterations 11,352 and 15,856.
Figures~\ref{fig:case:foot_basket} and \ref{fig:case:general_team} visualizes the corresponding hierarchical groupings.

\begin{figure}[t]
  \centering
  \includegraphics[width=0.99\columnwidth]{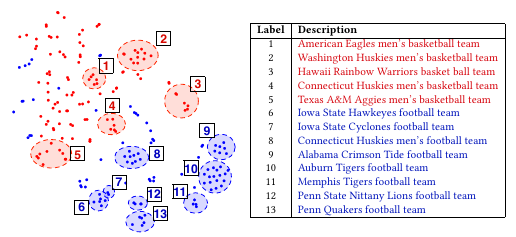}
  \caption{SFSes for USA basketball teams and USA football teams.}
  \label{fig:case:foot_basket}
\end{figure}

\begin{figure}[t]
  \centering
  \includegraphics[width=0.99\columnwidth]{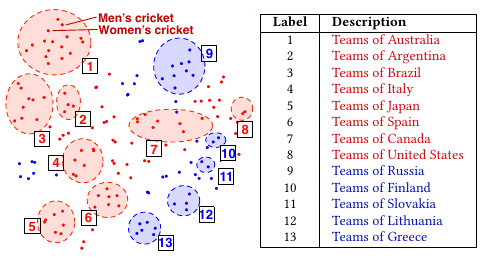}
  \caption{SFSes for sports teams from non-European and European countries.}
  \label{fig:case:general_team}
\end{figure}

\paragraph{Results and Analysis} 
%%% effective approximation
Figure~\ref{fig:exp2_iter} shows that the approximate Semantic Shift curve closely mirrors the exact one, achieving a Pearson correlation of \textbf{0.92}.
This strong correlation confirms the reliability of our efficient approximate method introduced in Section~\ref{sect:method:approximation}.
%%% Dynamic threshold
Moreover, \textbf{SAFARI}'s sliding-window-based dynamic thresholding effectively adapts to the non-uniform fluctuation of shifts across iterations, enable the discovery of subtle yet meaningful SFSes that would be missed by static thresholds.
%%% 
Further parameter study about the dynamic threshold mechanism in \textbf{SAFARI} is detailed in Appendix D.

\textbf{SAFARI} also reveals hierarchical semantic relationships with fine-to-coarse granularity.
%%%
In Figure \ref{fig:case:foot_basket}, early clusters capture specific U.S. university basketball and football teams, which gradually merge into broader categories like university sports teams.
%%%
Figure \ref{fig:case:general_team} illustrates cross-national grouping: teams initially cluster by country, then merge into regional groupings.
Notably, European teams (blue) consolidate more tightly than non-European teams (red), indicating structure-aware semantic abstraction.
%%%
These results showcase \textbf{SAFARI}'s ability to track evolving 
semantic organization, understanding latent hierarchies in embedding spaces.
Additional analysis is provided in Appendix E.

\section{Conclusions}
\label{sect:concl}

%%% problem
In this paper, we tackle the fundamental challenge of understanding the abstract and intricate structure of embedding spaces.
%%% methods
We introduce \textbf{SFSes} as a structured representation that explicitly links embedding spaces to their underlying semantics. 
Leveraging hierarchical clustering and the concept of Semantic Shift, we develop \textbf{SAFARI}, an effective algorithm that uncovers hierarchical semantic structures while maintaining computational scalability through an efficient approximation of Semantic Shift.
%%% experiments
Through comprehensive experiments on six real-world datasets spanning text and image modalities, we show that \textbf{SFSes} improve performance on both standard classification tasks and subtle semantic challenges like political bias detection.
\textbf{SAFARI} consistently reveals meaningful, modality-agnostic semantic hierarchies with minimal computational overhead.
%%% impacts
By bridging the gap between geometric embedding representations and their underlying semantics, this work opens new avenues for future research, like semantic-aware embedding analysis and knowledge discovery.

% \paragraph{Limitations and Broader Impact}
% %%% limitations
% Our hierarchical evaluations rely on LLM-generated labels, which may introduce bias or inconsistency from the language model.
% %%%
% While we use SVD to define subspaces precisely, practical deployments may benefit from learning transformation matrices to target SFSes.
% %%% broader impact
% Despite these limitations, our framework offers a general-purpose, interpretable approach for probing embedding structure. 
% %%%
% Understanding these structures can facilitate more controllable retrieval, enable post-hoc embedding refinements with off-the-shelf models, and guild future analyses of in-context learning by revealing geometric structure in activation space during inference.

% %% The acknowledgments section
% \begin{acks}
% \end{acks}

\bibliography{main.bib}

@misc{anthropic2025claude,
  author       = {Anthropic},
  title        = {Claude AI (v3.7)},
  note         = {Large language model},
  howpublished = {\url{https://www.anthropic.com}},
  year         = {2025}
}

@inproceedings{dhall2020hierarchical,
  title={Hierarchical image classification using entailment cone embeddings},
  author={Dhall, Ankit and Makarova, Anastasia and Ganea, Octavian and Pavllo, Dario and Greeff, Michael and Krause, Andreas},
  booktitle={Proceedings of the IEEE/CVF conference on Computer Vision and Pattern Recognition Workshops},
  pages={836--837},
  year={2020}
}

@inproceedings{deng2012hedging,
  title={Hedging your bets: Optimizing accuracy-specificity trade-offs in large scale visual recognition},
  author={Deng, Jia and Krause, Jonathan and Berg, Alexander C and Fei-Fei, Li},
  booktitle={2012 IEEE Conference on Computer Vision and Pattern Recognition (CVPR)},
  pages={3450--3457},
  year={2012}
}

@inproceedings{caron2020unsupervised,
  title={Unsupervised learning of visual features by contrasting cluster assignments},
  author={Caron, Mathilde and Misra, Ishan and Mairal, Julien and Goyal, Priya and Bojanowski, Piotr and Joulin, Armand},
  booktitle={Proceedings of the 34th International Conference on Neural Information Processing Systems (NeurIPS)},
  pages={9912--9924},
  year={2020}
}

@inproceedings{van2020scan,
  title={SCAN: Learning to Classify Images Without Labels},
  author={Van Gansbeke, Wouter and Vandenhende, Simon and Georgoulis, Stamatios and Proesmans, Marc and Van Gool, Luc},
  booktitle={European Conference on Computer Vision (ECCV)},
  pages={268--285},
  year={2020}
}

@article{goh2021multimodal,
  title={Multimodal Neurons in Artificial Neural Networks},
  author={Goh, Gabriel and Cammarata, Nick and Voss, Chelsea and Carter, Shan and Petrov, Michael and Schubert, Ludwig and Radford, Alec and Olah, Chris},
  journal={Distill},
  volume={6},
  number={3},
  pages={e30},
  year={2021}
}

@inproceedings{li2022supervision,
  title={Supervision Exists Everywhere: A Data Efficient Contrastive Language-Image Pre-training Paradigm},
  author={Li, Yangguang and Liang, Feng and Zhao, Lichen and Cui, Yufeng and Ouyang, Wanli and Shao, Jing and Yu, Fengwei and Yan, Junjie},
  booktitle={International Conference on Learning Representations (ICLR)},
  year={2022}
}

@inproceedings{jia2021scaling,
  title={Scaling up visual and vision-language representation learning with noisy text supervision},
  author={Jia, Chao and Yang, Yinfei and Xia, Ye and Chen, Yi-Ting and Parekh, Zarana and Pham, Hieu and Le, Quoc and Sung, Yun-Hsuan and Li, Zhen and Duerig, Tom},
  booktitle={International Conference on Machine Learning (ICML)},
  pages={4904--4916},
  year={2021}
}

@article{olah2017feature,
  title={Feature Visualization},
  author={Olah, Chris and Mordvintsev, Alexander and Schubert, Ludwig},
  journal={Distill},
  volume={2},
  number={11},
  pages={e7},
  year={2017}
}

@inproceedings{zhou2016learning,
  title={Learning deep features for discriminative localization},
  author={Zhou, Bolei and Khosla, Aditya and Lapedriza, Agata and Oliva, Aude and Torralba, Antonio},
  booktitle={Proceedings of the IEEE conference on Computer Vision and Pattern Recognition (CVPR)},
  pages={2921--2929},
  year={2016}
}

@inproceedings{wang2020understanding,
  title={Understanding contrastive representation learning through alignment and uniformity on the hypersphere},
  author={Wang, Tongzhou and Isola, Phillip},
  booktitle={International Conference on Machine Learning (ICML)},
  pages={9929--9939},
  year={2020}
}

@inproceedings{grill2020bootstrap,
  title={Bootstrap your own latent a new approach to self-supervised learning},
  author={Grill, Jean-Bastien and Strub, Florian and Altch{\'e}, Florent and Tallec, Corentin and Richemond, Pierre H and Buchatskaya, Elena and Doersch, Carl and Pires, Bernardo Avila and Guo, Zhaohan Daniel and Azar, Mohammad Gheshlaghi and others},
  booktitle={Proceedings of the 34th International Conference on Neural Information Processing Systems (NeurIPS)},
  pages={21271--21284},
  year={2020}
}

@inproceedings{he2020momentum,
  title={Momentum contrast for unsupervised visual representation learning},
  author={He, Kaiming and Fan, Haoqi and Wu, Yuxin and Xie, Saining and Girshick, Ross},
  booktitle={Proceedings of the IEEE/CVF conference on Computer Vision and Pattern Recognition (CVPR)},
  pages={9729--9738},
  year={2020}
}

@inproceedings{chen2020simple,
  title={A simple framework for contrastive learning of visual representations},
  author={Chen, Ting and Kornblith, Simon and Norouzi, Mohammad and Hinton, Geoffrey},
  booktitle={International Conference on Machine Learning (ICML)},
  pages={1597--1607},
  year={2020}
}

@inproceedings{isola2015discovering,
  title={Discovering states and transformations in image collections},
  author={Isola, Phillip and Lim, Joseph J and Adelson, Edward H},
  booktitle={Proceedings of the IEEE conference on Computer Vision and Pattern Recognition (CVPR)},
  pages={1383--1391},
  year={2015}
}

@book{meyer2000matrix,
  title={Matrix Analysis and Applied Linear Algebra},
  author={Meyer, Carl D},
  volume={71},
  year={2000},
  publisher={SIAM}
}

@book{belsley2005regression,
  title={Regression Diagnostics: Identifying Influential Data and Sources of Collinearity},
  author={Belsley, David A and Kuh, Edwin and Welsch, Roy E},
  year={2005},
  publisher={John Wiley \& Sons}
}

@article{breiman2001random,
  title={Random {F}orests},
  author={Breiman, Leo},
  journal={Machine Learning},
  volume={45},
  pages={5--32},
  year={2001}
}

@article{cover1967nearest,
  title={Nearest Neighbor Pattern Classification},
  author={Cover, Thomas and Hart, Peter},
  journal={IEEE Transactions on Information Theory},
  volume={13},
  number={1},
  pages={21--27},
  year={1967}
}

@inproceedings{karpukhin2020dense,
  title={Dense Passage Retrieval for Open-Domain Question Answering},
  author={Karpukhin, Vladimir and Oguz, Barlas and Min, Sewon and Lewis, Patrick and Wu, Ledell and Edunov, Sergey and Chen, Danqi and Yih, Wen-tau},
  booktitle={Proceedings of the 2020 Conference on Empirical Methods in Natural Language Processing (EMNLP)},
  pages={6769--6781},
  year={2020}
}

@book{fix1985discriminatory,
  title={Discriminatory Analysis. Nonparametric Discrimination: Consistency Properties},
  author={Fix, Evelyn},
  volume={1},
  year={1985},
  publisher={USAF School of Aviation Medicine}
}

@article{platt1999probabilistic,
  title={Probabilistic Outputs for Support Vector Machines and Comparisons to Regularized Likelihood Methods},
  author={Platt, John},
  journal={Advances in Large Margin Classifiers},
  volume={10},
  number={3},
  pages={61--74},
  year={1999}
}

@article{chang2011libsvm,
  title={{LIBSVM}: A {L}ibrary for {S}upport {V}ector {M}achines},
  author={Chang, Chih-Chung and Lin, Chih-Jen},
  journal={ACM Transactions on Intelligent Systems and Technology (TIST)},
  volume={2},
  number={3},
  pages={1--27},
  year={2011}
}

@inproceedings{clark2019does,
  title={What Does {BERT} Look at? An Analysis of {BERT}'s Attention},
  author={Clark, Kevin and Khandelwal, Urvashi and Levy, Omer and Manning, Christopher D},
  booktitle={Proceedings of the 2019 ACL Workshop BlackboxNLP: Analyzing and Interpreting Neural Networks for NLP},
  pages={276--286},
  year={2019}
}

@inproceedings{ethayarajh2019contextual,
  title={How Contextual are Contextualized Word Representations? {C}omparing the Geometry of {BERT}, {ELMo}, and {GPT-2} Embeddings},
  author={Ethayarajh, Kawin},
  booktitle={Proceedings of the 2019 Conference on Empirical Methods in Natural Language Processing and the 9th International Joint Conference on Natural Language Processing (EMNLP-IJCNLP)},
  pages={55--65},
  year={2019}
}

@inproceedings{dalvi2019one,
  title={What Is One Grain of Sand in the Desert? {A}nalyzing Individual Neurons in Deep {NLP} Models},
  author={Dalvi, Fahim and Durrani, Nadir and Sajjad, Hassan and Belinkov, Yonatan and Bau, Anthony and Glass, James},
  booktitle={Proceedings of the AAAI Conference on Artificial Intelligence (AAAI)},
  pages={6309--6317},
  year={2019}
}

@inproceedings{simhi2023interpreting,
  title={Interpreting Embedding Spaces by Conceptualization},
  author={Simhi, Adi and Markovitch, Shaul},
  booktitle={Proceedings of the 2023 Conference on Empirical Methods in Natural Language Processing (EMNLP)},
  pages={1704--1719},
  year={2023}
}

@inproceedings{dufter2019analytical,
  title={Analytical Methods for Interpretable Ultradense Word Embeddings},
  author={Dufter, Philipp and Sch{\"u}tze, Hinrich},
  booktitle={Proceedings of the 2019 Conference on Empirical Methods in Natural Language Processing and the 9th International Joint Conference on Natural Language Processing (EMNLP-IJCNLP)},
  pages={1185--1191},
  year={2019}
}

@inproceedings{park2017rotated,
  title={Rotated word vector representations and their interpretability},
  author={Park, Sungjoon and Bak, JinYeong and Oh, Alice},
  booktitle={Proceedings of the 2017 Conference on Empirical Methods in Natural Language Processing (EMNLP)},
  pages={401--411},
  year={2017}
}

@inproceedings{mu2018all,
  title={{All-but-the-Top}: Simple and Effective Postprocessing for Word Representations},
  author={Mu, Jiaqi and Viswanath, Pramod},
  booktitle={International Conference on Learning Representations (ICLR)},
  year={2018}
}

@inproceedings{liu2019unsupervised,
  title={Unsupervised post-processing of word vectors via conceptor negation},
  author={Liu, Tianlin and Ungar, Lyle and Sedoc, Joao},
  booktitle={Proceedings of the AAAI Conference on Artificial Intelligence (AAAI)},
  pages={6778--6785},
  year={2019}
}

@inproceedings{Mikolov2013Word2Vec,
  title={Distributed representations of words and phrases and their compositionality},
  author={Mikolov, Tomas and Sutskever, Ilya and Chen, Kai and Corrado, Greg and Dean, Jeffrey},
  booktitle={Proceedings of the 26th International Conference on Neural Information Processing Systems (NIPS)},
  pages={3111--3119},
  year={2013}
}

@inproceedings{devlin2019bert,
  title={{BERT}: {P}re-training of Deep Bidirectional Transformers for Language Understanding},
  author={Devlin, Jacob and Chang, Ming-Wei and Lee, Kenton and Toutanova, Kristina},
  booktitle={Proceedings of the 2019 Conference of the North American Chapter of the Association for Computational Linguistics: Human Language Technologies (NAACL-HLT)},
  pages={4171--4186},
  year={2019}
}

@article{weyl1912asymptotische,
  title={Das asymptotische Verteilungsgesetz der Eigenwerte linearer partieller Differentialgleichungen (mit einer Anwendung auf die Theorie der Hohlraumstrahlung)},
  author={Weyl, Hermann},
  journal={Mathematische Annalen},
  volume={71},
  number={4},
  pages={441--479},
  year={1912}
}

@inproceedings{zhang2015character,
  title={Character-level convolutional networks for text classification},
  author={Zhang, Xiang and Zhao, Junbo and LeCun, Yann},
  booktitle={Proceedings of the 28th International Conference on Neural Information Processing Systems (NIPS)},
  pages={649--657},
  year={2015}
}

@inproceedings{yang2018sgm,
  title={{SGM:} {S}equence Generation Model for Multi-label Classification},
  author={Yang, Pengcheng and Sun, Xu and Li, Wei and Ma, Shuming and Wu, Wei and Wang, Houfeng},
  booktitle={Proceedings of the 27th International Conference on Computational Linguistics (COLING)},
  pages={3915--3926},
  year={2018}
}

@inproceedings{maas2011learning,
  title={Learning Word Vectors for Sentiment Analysis},
  author={Maas, Andrew and Daly, Raymond E and Pham, Peter T and Huang, Dan and Ng, Andrew Y and Potts, Christopher},
  booktitle={Proceedings of the 49th Annual Meeting of the Association for Computational Linguistics (ACL)},
  pages={142--150},
  year={2011}
}

@inproceedings{wu2020scalable,
  title={Scalable Zero-shot Entity Linking with Dense Entity Retrieval},
  author={Wu, Ledell and Petroni, Fabio and Josifoski, Martin and Riedel, Sebastian and Zettlemoyer, Luke},
  booktitle={Proceedings of the 2020 Conference on Empirical Methods in Natural Language Processing (EMNLP)},
  pages={6397--6407},
  year={2020}
}

@article{stewart1998perturbation,
  title={{P}erturbation {T}heory for the {S}ingular {V}alue {D}ecomposition},
  author={Stewart, Gilbert W},
  year={1998}
}

@article{halko2009finding,
  title={Finding structure with randomness: Stochastic algorithms for constructing approximate matrix decompositions},
  author={Halko, Nathan and Martinsson, Per-Gunnar and Tropp, Joel A},
  journal={arXiv preprint arXiv:0909.4061},
  year={2009}
}

@book{trefethen2022numerical,
  title={{N}umerical {L}inear {A}lgebra},
  author={Trefethen, Lloyd N and Bau, David},
  year={2022},
  publisher={SIAM}
}

@book{schutze2008introduction,
  title={{I}ntroduction to {I}nformation {R}etrieval},
  author={Sch{\"u}tze, Hinrich and Manning, Christopher D and Raghavan, Prabhakar},
  year={2008},
  publisher={Cambridge University Press}
}

@book{leskovec2020mining,
  title={{M}ining of {M}assive {D}atasets},
  author={Leskovec, Jure and Rajaraman, Anand and Ullman, Jeffrey David},
  year={2020},
  publisher={Cambridge University Press}
}

@inproceedings{radford2021learning,
  title={Learning transferable visual models from natural language supervision},
  author={Radford, Alec and Kim, Jong Wook and Hallacy, Chris and Ramesh, Aditya and Goh, Gabriel and Agarwal, Sandhini and Sastry, Girish and Askell, Amanda and Mishkin, Pamela and Clark, Jack and others},
  booktitle={International Conference on Machine Learning (ICML)},
  pages={8748--8763},
  year={2021}
}

@article{ullmann1957principles,
  title={The Principles of Semantics},
  author={Ullmann, Stephen},
  journal={},
  year={1957}
}

@book{nida2015componential,
  title={A componential analysis of meaning: An introduction to semantic structures},
  author={Nida, Eugene A},
  volume={57},
  year={2015},
  publisher={Walter de Gruyter GmbH \& Co KG}
}

@inproceedings{mimno2017strange,
  title={The strange geometry of skip-gram with negative sampling},
  author={Mimno, David and Thompson, Laure},
  booktitle={Proceedings of the 2017 Conference on Empirical Methods in Natural Language Processing (EMNLP)},
  pages={2873--2878},
  year={2017}
}

@inproceedings{lewis2020retrieval,
  title={Retrieval-{A}ugmented {G}eneration for {K}nowledge-{I}ntensive {NLP} {T}asks},
  author={Lewis, Patrick and Perez, Ethan and Piktus, Aleksandra and Petroni, Fabio and Karpukhin, Vladimir and Goyal, Naman and K{\"u}ttler, Heinrich and Lewis, Mike and Yih, Wen-tau and Rockt{\"a}schel, Tim and others},
  booktitle={Proceedings of the 34th International Conference on Neural Information Processing Systems (NeurIPS)},
  pages={9459--9474},
  year={2020}
}

@inproceedings{guu2020realm,
  title={{REALM:} {R}etrieval-{A}ugmented {L}anguage {M}odel {P}re-{T}raining},
  author={Guu, Kelvin and Lee, Kenton and Tung, Zora and Pasupat, Panupong and Chang, Ming-Wei},
  booktitle={Proceedings of the 37th International Conference on Machine Learning (ICML)},
  pages={3929--3938},
  year={2020}
}

@article{ram2023context,
  title = {In-Context Retrieval-Augmented Language Models},
  author = {Ram, Ori and Levine, Yoav and Dalmedigos, Itay and Muhlgay, Dor and Shashua, Amnon and Leyton-Brown, Kevin and Shoham, Yoav},
  journal = {Transactions of the Association for Computational Linguistics (TACL)},
  pages = {1316--1331},
  volume = {11},
  year = {2023}
}

@inproceedings{asai2024self,
  title={{Self-RAG}: Learning to Retrieve, Generate, and Critique through Self-Reflection},
  author={Asai, Akari and Wu, Zeqiu and Wang, Yizhong and Sil, Avirup and Hajishirzi, Hannaneh},
  booktitle={International Conference on Learning Representations (ICLR)},
  year={2024}
}

@article{sun2024diversinews,
  title={{DiversiNews}: Enriching News Consumption with Relevant Yet Diverse News Articles Retrieval},
  author={Sun, Yiqun and Huang, Qiang and Wang, Yanhao and Tung, Anthony KH},
  journal={Proceedings of the VLDB Endowment},
  volume={17},
  number={12},
  pages={4277--4280},
  year={2024}
}

@inproceedings{hirata2022solving,
  title = {Solving Diversity-Aware Maximum Inner Product Search Efficiently and Effectively},
  author = {Hirata, Kohei and Amagata, Daichi and Fujita, Sumio and Hara, Takahiro},
  booktitle = {Proceedings of the 16th ACM Conference on Recommender Systems (RecSys)},
  pages = {198--207},
  year = {2022}
}

@inproceedings{gan2020enhancing,
  title={Enhancing recommendation diversity using determinantal point processes on knowledge graphs},
  author={Gan, Lu and Nurbakova, Diana and Laporte, L{\'e}a and Calabretto, Sylvie},
  booktitle={Proceedings of the 43rd International ACM SIGIR Conference on Research and Development in Information Retrieval (SIGIR)},
  pages={2001--2004},
  year={2020}
}

@article{huang2024diversity,
  title={Diversity-Aware $k$-Maximum Inner Product Search Revisited},
  author={Huang, Qiang and Wang, Yanhao and Sun, Yiqun and Tung, Anthony KH},
  journal={arXiv preprint arXiv:2402.13858},
  year={2024}
}

@article{yu2019multimodal,
  title={Multimodal transformer with multi-view visual representation for image captioning},
  author={Yu, Jun and Li, Jing and Yu, Zhou and Huang, Qingming},
  journal={IEEE Transactions on Circuits and Systems for Video Technology (TCSVT)},
  volume={30},
  number={12},
  pages={4467--4480},
  year={2019}
}

@inproceedings{luo2023semantic,
  title={Semantic-conditional diffusion networks for image captioning},
  author={Luo, Jianjie and Li, Yehao and Pan, Yingwei and Yao, Ting and Feng, Jianlin and Chao, Hongyang and Mei, Tao},
  booktitle={Proceedings of the IEEE/CVF Conference on Computer Vision and Pattern Recognition (CVPR)},
  pages={23359--23368},
  year={2023}
}

@inproceedings{yu2023self,
  title={Self-chained image-language model for video localization and question answering},
  author={Yu, Shoubin and Cho, Jaemin and Yadav, Prateek and Bansal, Mohit},
  booktitle={Proceedings of the 37th International Conference on Neural Information Processing Systems (NeurIPS)},
  pages={76749--76771},
  year={2023}
}

@inproceedings{zhang2024learnability,
  title={Learnability Matters: Active Learning for Video Captioning},
  author={Zhang, Yiqian and Liu, Buyu and Bao, Jun and Huang, Qiang and Zhang, Min and Yu, Jun},
  booktitle={Proceedings of the 38th Annual Conference on Neural Information Processing Systems (NeurIPS)},
  pages={37928--37954},
  year={2024}
}

@inproceedings{li2024aoe,
  title={AoE: Angle-optimized embeddings for semantic textual similarity},
  author={Li, Xianming and Li, Jing},
  booktitle={Proceedings of the 62nd Annual Meeting of the Association for Computational Linguistics (ACL)},
  pages={1825--1839},
  year={2024}
}

@inproceedings{muennighoff2023mteb,
  title={MTEB: Massive Text Embedding Benchmark},
  author={Muennighoff, Niklas and Tazi, Nouamane and Magne, Loic and Reimers, Nils},
  booktitle={Proceedings of the 17th Conference of the European Chapter of the Association for Computational Linguistics (EACL)},
  pages={2014--2037},
  year={2023}
}

@incollection{hinton1990connectionist,
  title={Connectionist learning procedures},
  author={Hinton, Geoffrey E},
  booktitle={Machine Learning},
  pages={555--610},
  year={1990},
  publisher={Elsevier}
}

@inproceedings{he2015delving,
  title={Delving deep into rectifiers: Surpassing human-level performance on imagenet classification},
  author={He, Kaiming and Zhang, Xiangyu and Ren, Shaoqing and Sun, Jian},
  booktitle={Proceedings of the IEEE International Conference on Computer Vision (ICCV)},
  pages={1026--1034},
  year={2015}
}
\makeatletter
\@ifundefined{isChecklistMainFile}{
  % We are compiling a standalone document
  \newif\ifreproStandalone
  \reproStandalonetrue
}{
  % We are being \input into the main paper
  \newif\ifreproStandalone
  \reproStandalonefalse
}
\makeatother

\ifreproStandalone
\documentclass[letterpaper]{article}
\usepackage[submission]{aaai2026}
\setlength{\pdfpagewidth}{8.5in}
\setlength{\pdfpageheight}{11in}
\usepackage{times}
\usepackage{helvet}
\usepackage{courier}
\usepackage{xcolor}
\frenchspacing

\begin{document}
\fi
\setlength{\leftmargini}{20pt}
\makeatletter\def\@listi{\leftmargin\leftmargini \topsep .5em \parsep .5em \itemsep .5em}
\def\@listii{\leftmargin\leftmarginii \labelwidth\leftmarginii \advance\labelwidth-\labelsep \topsep .4em \parsep .4em \itemsep .4em}
\def\@listiii{\leftmargin\leftmarginiii \labelwidth\leftmarginiii \advance\labelwidth-\labelsep \topsep .4em \parsep .4em \itemsep .4em}\makeatother

\setcounter{secnumdepth}{0}
\renewcommand\thesubsection{\arabic{subsection}}
\renewcommand\labelenumi{\thesubsection.\arabic{enumi}}

\newcounter{checksubsection}
\newcounter{checkitem}[checksubsection]

\newcommand{\checksubsection}[1]{%
  \refstepcounter{checksubsection}%
  \paragraph{\arabic{checksubsection}. #1}%
  \setcounter{checkitem}{0}%
}

\newcommand{\checkitem}{%
  \refstepcounter{checkitem}%
  \item[\arabic{checksubsection}.\arabic{checkitem}.]%
}
\newcommand{\question}[2]{\normalcolor\checkitem #1 #2 \color{blue}}
\newcommand{\ifyespoints}[1]{\makebox[0pt][l]{\hspace{-15pt}\normalcolor #1}}

\section*{Reproducibility Checklist}

\vspace{1em}
\hrule
\vspace{1em}

\textbf{Instructions for Authors:}

This document outlines key aspects for assessing reproducibility. Please provide your input by editing this \texttt{.tex} file directly.

For each question (that applies), replace the ``Type your response here'' text with your answer.

\vspace{1em}
\noindent
\textbf{Example:} If a question appears as

\begin{center}
\noindent
\begin{minipage}{.9\linewidth}
\ttfamily\raggedright
\string\question \{Proofs of all novel claims are included\} \{(yes/partial/no)\} \\
Type your response here
\end{minipage}
\end{center}
you would change it to:
\begin{center}
\noindent
\begin{minipage}{.9\linewidth}
\ttfamily\raggedright
\string\question \{Proofs of all novel claims are included\} \{(yes/partial/no)\} \\
yes
\end{minipage}
\end{center}

Please make sure to:
\begin{itemize}\setlength{\itemsep}{.1em}
\item Replace ONLY the ``Type your response here'' text and nothing else.
\item Use one of the options listed for that question (e.g., \textbf{yes}, \textbf{no}, \textbf{partial}, or \textbf{NA}).
\item \textbf{Not} modify any other part of the \texttt{\string\question} command or any other lines in this document.\\
\end{itemize}

You can \texttt{\string\input} this .tex file right before \texttt{\string\end\{document\}} of your main file or compile it as a stand-alone document. Check the instructions on your conference's website to see if you will be asked to provide this checklist with your paper or separately.

\vspace{1em}
\hrule
\vspace{1em}

% The questions start here

\checksubsection{General Paper Structure}
\begin{itemize}

\question{Includes a conceptual outline and/or pseudocode description of AI methods introduced}{(yes/partial/no/NA)}
\textbf{yes}

\question{Clearly delineates statements that are opinions, hypothesis, and speculation from objective facts and results}{(yes/no)}
\textbf{yes}

\question{Provides well-marked pedagogical references for less-familiar readers to gain background necessary to replicate the paper}{(yes/no)}
\textbf{yes}

\end{itemize}
\checksubsection{Theoretical Contributions}
\begin{itemize}

\question{Does this paper make theoretical contributions?}{(yes/no)}
\textbf{yes}

	\ifyespoints{\vspace{1.2em}If yes, please address the following points:}
        \begin{itemize}
	
	\question{All assumptions and restrictions are stated clearly and formally}{(yes/partial/no)}
    \textbf{yes}

	\question{All novel claims are stated formally (e.g., in theorem statements)}{(yes/partial/no)}
    \textbf{yes}

	\question{Proofs of all novel claims are included}{(yes/partial/no)}
    \textbf{yes}

	\question{Proof sketches or intuitions are given for complex and/or novel results}{(yes/partial/no)}
    \textbf{yes}

	\question{Appropriate citations to theoretical tools used are given}{(yes/partial/no)}
    \textbf{yes}

	\question{All theoretical claims are demonstrated empirically to hold}{(yes/partial/no/NA)}
    \textbf{yes}

	\question{All experimental code used to eliminate or disprove claims is included}{(yes/no/NA)}
    \textbf{yes}
	
	\end{itemize}
\end{itemize}

\checksubsection{Dataset Usage}
\begin{itemize}

\question{Does this paper rely on one or more datasets?}{(yes/no)}
\textbf{yes}

\ifyespoints{If yes, please address the following points:}
\begin{itemize}

	\question{A motivation is given for why the experiments are conducted on the selected datasets}{(yes/partial/no/NA)}
	\textbf{yes}

	\question{All novel datasets introduced in this paper are included in a data appendix}{(yes/partial/no/NA)}
	\textbf{NA}

	\question{All novel datasets introduced in this paper will be made publicly available upon publication of the paper with a license that allows free usage for research purposes}{(yes/partial/no/NA)}
	\textbf{NA}

	\question{All datasets drawn from the existing literature (potentially including authors' own previously published work) are accompanied by appropriate citations}{(yes/no/NA)}
	\textbf{yes}

	\question{All datasets drawn from the existing literature (potentially including authors' own previously published work) are publicly available}{(yes/partial/no/NA)}
	\textbf{yes}

	\question{All datasets that are not publicly available are described in detail, with explanation why publicly available alternatives are not scientifically satisficing}{(yes/partial/no/NA)}
	\textbf{NA}

\end{itemize}
\end{itemize}

\checksubsection{Computational Experiments}
\begin{itemize}

\question{Does this paper include computational experiments?}{(yes/no)}
\textbf{yes}

\ifyespoints{If yes, please address the following points:}
\begin{itemize}

	\question{This paper states the number and range of values tried per (hyper-) parameter during development of the paper, along with the criterion used for selecting the final parameter setting}{(yes/partial/no/NA)}
	\textbf{yes}

	\question{Any code required for pre-processing data is included in the appendix}{(yes/partial/no)}
	\textbf{yes}

	\question{All source code required for conducting and analyzing the experiments is included in a code appendix}{(yes/partial/no)}
	\textbf{yes}

	\question{All source code required for conducting and analyzing the experiments will be made publicly available upon publication of the paper with a license that allows free usage for research purposes}{(yes/partial/no)}
	\textbf{yes}
        
	\question{All source code implementing new methods have comments detailing the implementation, with references to the paper where each step comes from}{(yes/partial/no)}
	\textbf{yes}

	\question{If an algorithm depends on randomness, then the method used for setting seeds is described in a way sufficient to allow replication of results}{(yes/partial/no/NA)}
	\textbf{NA}

	\question{This paper specifies the computing infrastructure used for running experiments (hardware and software), including GPU/CPU models; amount of memory; operating system; names and versions of relevant software libraries and frameworks}{(yes/partial/no)}
	\textbf{yes}

	\question{This paper formally describes evaluation metrics used and explains the motivation for choosing these metrics}{(yes/partial/no)}
	\textbf{yes}

	\question{This paper states the number of algorithm runs used to compute each reported result}{(yes/no)}
	\textbf{yes}

	\question{Analysis of experiments goes beyond single-dimensional summaries of performance (e.g., average; median) to include measures of variation, confidence, or other distributional information}{(yes/no)}
	\textbf{no}

	\question{The significance of any improvement or decrease in performance is judged using appropriate statistical tests (e.g., Wilcoxon signed-rank)}{(yes/partial/no)}
	\textbf{no}

	\question{This paper lists all final (hyper-)parameters used for each model/algorithm in the paper’s experiments}{(yes/partial/no/NA)}
	\textbf{yes}

\end{itemize}
\end{itemize}
\ifreproStandalone
\end{document}
\fi
\appendix

\setcounter{table}{0}
\renewcommand{\thetable}{\thesection\arabic{table}}
\setcounter{theorem}{0}
\renewcommand{\thetheorem}{\thesection\arabic{theorem}}
\setcounter{equation}{0}
\renewcommand{\theequation}{\thesection\arabic{equation}}
\setcounter{figure}{0}
\renewcommand{\thefigure}{\thesection\arabic{figure}}

\section{Experiment Details}
\label{appendix:expt-details}

\subsection{Dataset Details}
\label{appendix:expt-details:datasets}

We evaluate \textbf{SAFARI} on six widely used real-world datasets spanning both text and image modalities, demonstrating its effectiveness and generalizability across domains.
%%%
For the text modality, we use five diverse datasets:
\begin{itemize}
  \item \textbf{AG-News}~\cite{zhang2015character} consists of over 1 million articles from 2,000+ sources, categorized into four topics: \textbf{Business}, \textbf{Sci/Tech}, \textbf{Sports}, and \textbf{World}. 
  This serves as our primary dataset due to its scale and well-defined semantic structures.

  \item \textbf{AAPD}~\cite{yang2018sgm} contains 55,840 arXiv abstracts from computer science, each labeled with one or more subject areas, supporting multi-label classification tasks.
  
  \item \textbf{IMDB}~\cite{maas2011learning} comprises 50,000 movie reviews split evenly into training and testing sets, designed for binary sentiment classification.\footnote{\url{https://www.imdb.com/}}
  
  \item \textbf{Yelp} includes 6.99 million user reviews with 150,000+ business attributes, enabling fine-grained semantic analysis.\footnote{\url{https://www.yelp.com/dataset}}

  \item \textbf{NewsSpectrum}~\cite{sun2024diversinews} offers 250,000 politically diverse news articles from Reddit. It offers a balanced distribution across the ideological spectrum, making it well-suited for studying abstract semantic phenomena such as political bias.
\end{itemize}

For the image modality, we use \textbf{MIT-States} dataset~\citep{isola2015discovering}, which comprises $\sim$53,000 images labeled with 245 object classes and 115 attribute states. 
It is specifically designed to evaluate models on object-attribute compositions and compositional generalization.

\begin{figure}[t]
  \centering
  \includegraphics[width=0.99\columnwidth]{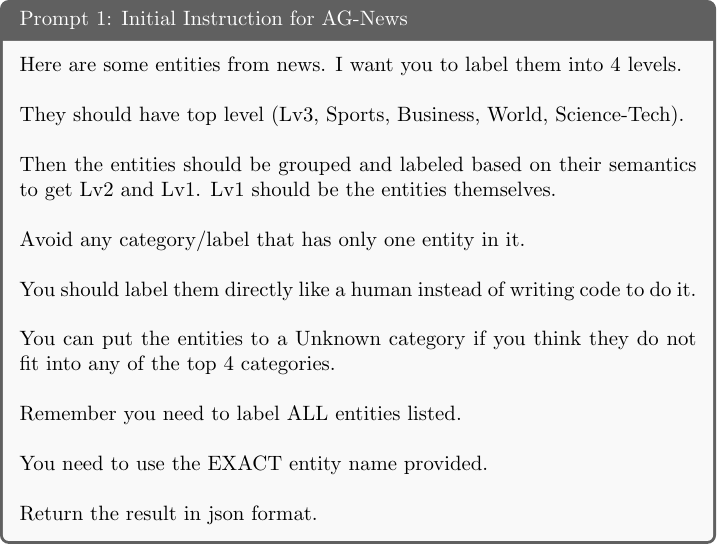}
  \caption{Initial prompt to generate a partial AG-News hierarchy with fixed top-level labels.} 
  \label{fig:prompts_ag_news_1}
\end{figure}

\begin{figure}[t]
  \centering
  \includegraphics[width=0.99\columnwidth]{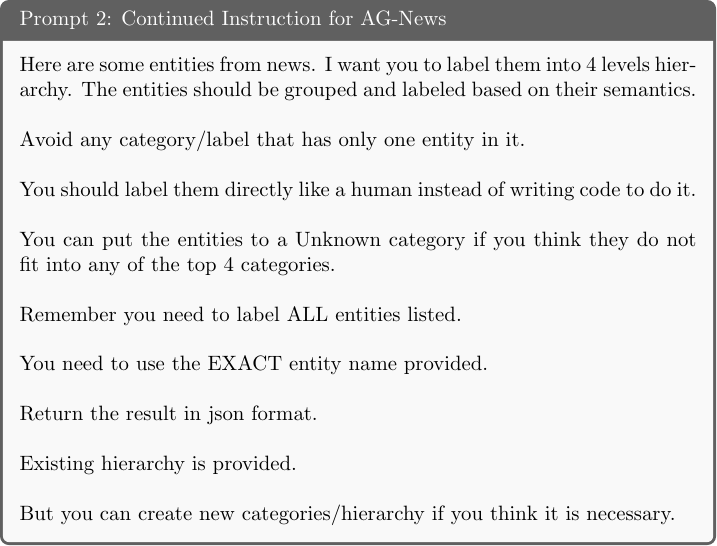}
  \caption{Second-stage prompt expanding the AG-News hierarchy with additional entity subsets.} 
  \label{fig:prompts_ag_news_2}
\end{figure}

\subsection{Hierarchical Label Structures For AG-News and MIT-States}
\label{appendix:expt-details:hierarchy}

To evaluate semantic coherence across multiple granularities, we construct four-level hierarchies for both datasets:
\begin{itemize}
  \item \textbf{AG-News:} Category $\rightarrow$ Subcategory $\rightarrow$ Semantic Group $\rightarrow$ Entity; 
  
  \item \textbf{MIT-States:} Category $\rightarrow$ Subcategory $\rightarrow$ Object $\rightarrow$ Attribute. 
\end{itemize}

These hierarchies capture increasingly specific relationships, ranging from broad domains to specific entities in \textbf{AG-News}, as well as conceptual relationships such as \term{Materials \& Substances $\rightarrow$ Metals $\rightarrow$ Steel $\rightarrow$ Unpainted} in \textbf{MIT-States}. 

\paragraph{Prompts for AG-News}
We employ a two-stage prompting strategy using Claude 3.7 Sonnet to construct the AG-News hierarchy.
%%%
Directly prompting with the full entity list often causes omission due to context length limitations.
To mitigate this, we begin with a small subset of entities and prompt Claude to generate a hierarchy with fixed top-level categories (Figure \ref{fig:prompts_ag_news_1}).
Then, we iteratively expand the hierarchy by combining previous outputs with new subsets of entities (Figure \ref{fig:prompts_ag_news_2}), until all entities are processed.

\begin{figure}[t]
  \centering
  \includegraphics[width=0.99\columnwidth]{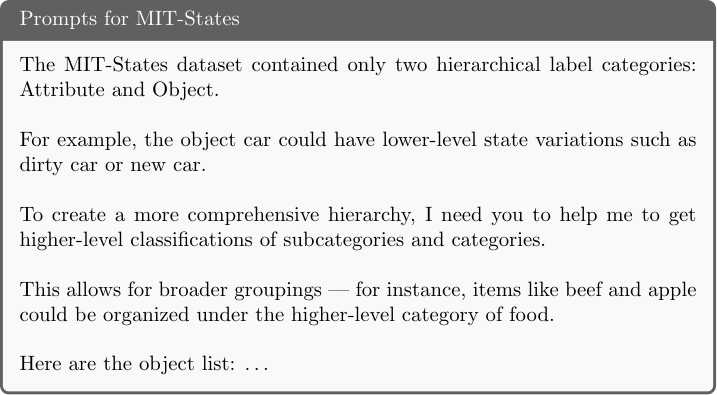}
  \caption{Prompt template for constructing MIT-States hierarchies via object-attribute refinement.} 
  \label{fig:prompts_mit_states}
\end{figure}

\paragraph{Prompts for MIT-States}
For MIT-States, we filter out images and categories that lack hierarchical depth, either due to having only one image (e.g., \term{dog}, \term{car}) or missing parent categories.
The prompt used to construct valid hierarchical relationships is shown in Figure \ref{fig:prompts_mit_states}.

\paragraph{Hierarchical Labels for AG-News and MIT-States}
The complete hierarchical label mappings for AG-News and MIT-States are provided in Tables \ref{tab:agnews-labels} and \ref{tab:mit-categories}, respectively.
These serve as ground truth for evaluating semantic coherence at different levels of abstraction.

\begin{table*}[t]
\small
\centering
\renewcommand{\arraystretch}{1.0}
\caption{Hierarchical labels for AG-News dataset.}
\label{tab:agnews-labels}
\resizebox{0.99\textwidth}{!}{
\begin{tabular}{p{2.0cm} p{3.5cm} p{10.5cm}}
  \toprule
  \textbf{Category} & \textbf{Subcategory} & \textbf{Semantic Group} \\
  \midrule
  \multirow{9}{*}{Sports} 
  & Olympic Sports & Teams \& Organizations, Events \& Competitions, Athletes, Venues \\
  & Basketball & Teams \& Seasons, Players \& Personnel, Venues, Organizations \& Events \\
  & American Football & Teams \& Seasons, Players \& Personnel, Venues \& Concepts \\
  & Baseball & Teams \& Seasons, Players \& Personnel, Events \& Concepts \\
  & Other Sports & Golf \& Tennis, Racing \& Motorsports, Combat Sports, Rugby \& Cricket, Horse Racing, Soccer \& Football, Swimming \& Water Sports, Other Sporting Events, Other Sports Personnel, Winter Sports, Teams \& Organizations, Players \& Personnel, Venues \& Events, Cricket \& Rugby, Other Sports Events, Other Sports \\
  & Soccer \& Football & Teams \& Organizations, Players \& Personnel, Venues \& Events \\
  \midrule
  
  \multirow{9}{*}{Business}
  & Financial Services & Banking, Investment \& Asset Management, Insurance \& Risk Management, Consulting \& Advisory \\
  & Corporations \& Industries & Manufacturing \& Industrial, Retail \& Consumer Goods, Automotive, Transportation \& Logistics, Energy \& Resources, Technology \& Telecommunications \\
  & Technology Companies & Software \& IT, Security \& Cybersecurity, Media \& Entertainment, Hardware \& Computing, Technology Services \\
  & Retail \& Consumer Goods & Retail Companies, Food \& Beverage, Marketing \& Advertising \\
  & Corporate Entities & Corporations \& Conglomerates, Executives \& Entrepreneurs \\
  \midrule

  \multirow{6}{*}{World}
  & Politics \& Government & Political Figures, Government Organizations, Political Events \& Issues, International Relations, Political Movements \\
  & Cultural \& Social & Literature \& Writers, Social Groups, Media \& Entertainment Figures, Arts \& Culture, Social Issues \\
  & Law \& Justice & Legal Cases \& Legislation, Legal Professionals, Crime \& Legal Issues \\
  & Military \& Security & Military Conflicts, Organizations, Personnel, Security \& Intelligence \\
  \midrule

  \multirow{5}{*}{Science-Tech}
  & Space \& Astronomy & Space Exploration, Astronomical Research, Astronomical Objects \\
  & Computing \& Technology & Software \& Development, Hardware \& Devices, Internet \& Telecom, Digital Media, IT Infrastructure \\
  & Medical \& Health & Medical Technology, Research, Healthcare Organizations \\
  & Environmental Science & Climate \& Earth Sciences, Environmental Events \\
  \bottomrule
\end{tabular}}
\end{table*}
\begin{table*}[ht]
\small
\centering
\renewcommand{\arraystretch}{1.0}
\caption{Hierarchical labels for the MIT-States dataset.}
\label{tab:mit-categories}
% \resizebox{0.99\textwidth}{!}{
\begin{tabular}{p{3.3cm} p{3.7cm} p{9.5cm}} % lll
\toprule
\textbf{Category} & \textbf{Subcategory} & \textbf{Object} \\
\midrule
\multirow{3}{*}{Materials and Substances} 
& Metals & Aluminum, Brass, Bronze, Copper, Metal, Steel \\
& Natural Materials & Clay, Cotton, Fabric, Foam, Paper, Paste, Plastic, Silk, Velvet, Wool \\
& Earth Elements & Concrete, Dirt, Granite, Ground, Mud, Rock, Sand, Stone \\
\midrule
\multirow{3}{*}{Food and Consumables} 
& Proteins & Beef, Chicken, Fish, Meat, Salmon, Seafood \\
& Produce & Apple, Fruit, Pear, Tomato, Vegetable, Potato \\
& Prepared Foods & Bread, Cheese, Cookie, Eggs, Pie, Pizza, Soup \\
\midrule
\multirow{4}{*}{Built Environment} 
& Structures & Building, Castle, Church, House, Wall \\
& Spaces & Bathroom, Kitchen, Room \\
& Furniture and Fixtures & Cabinet, Chair, Lightbulb, Tile \\
& Transportation Infrastructure & Highway, Road, Street \\
\midrule
\multirow{5}{*}{Nature} 
& Bodies of Water & Lake, Pond, Pool \\
& Landforms & Canyon, Valley \\
& Flora & Forest, Plant, Redwood, Tree \\
& Sky Elements & Cloud, Sky \\
& Agricultural & Farm \\
\midrule
\multirow{2}{*}{Consumer Goods} 
& Clothing and Accessories & Bracelet, Clothes, Coat, Dress, Necklace, Pants, Ribbon, Ring, Shirt, Shorts \\
& Household Items & Bag, Blade, Bottle, Camera, Carpet, Clock, Glass, Knife, Rope, Toy \\
\bottomrule
\end{tabular}%}
\end{table*}

\begin{table*}[t]
\centering
\small
\renewcommand{\arraystretch}{1.0}
\setlength\tabcolsep{6pt}
\caption{Detailed experiment settings.}
\label{tab:expt-settings}
\resizebox{0.99\textwidth}{!}{
  \begin{tabular}{p{2.8cm} p{2.5cm} p{1.8cm} p{3.2cm} p{3.0cm} p{3.0cm}} % 
  \toprule
  \textbf{Experiment} & \textbf{Datasets} & \textbf{Models / Embeddings} & \textbf{Parameters} & \textbf{Metrics} & \textbf{Baselines} \\ 
  \midrule
  Hierarchical Structure Discovery in Text & AG-News & BLINK & Dynamic thresholding with sliding window & Impurity & N/A \\ 
  \midrule
  Hierarchical Structure Discovery in Image & MIT-States (53K images) & CLIP & Dynamic thresholding with sliding window & Impurity & N/A \\
  \midrule
  Topic Classification & AG-News, AAPD, IMDB, Yelp & BLINK & Top-n entities & Precision, Recall, F1-score     & SVM, KNN, RF, MLP  \\
  \midrule
  Image Classification & MIT-States & CLIP & Top-n dimensions (5\%)       & Precision, Recall, F1-score     & SVM, KNN, RF, MLP  \\
  \midrule
  Computational Efficiency   & AG-News Sampled  & BLINK  & Dynamic thresholding with sliding window & Runtime, Average error             & SAFARI (Full SVD) \\
  \midrule
  Political Bias Detection & NewsSpectrum    & AnglE  & Top-n entities   & Runtime, F1-score & SVM, KNN, RF, MLP  \\ 
  \midrule
  Component Analysis of Semantic Shift & AG-News Sampled   & AnglE        & N/A  & DIS/DC ratio, Mean, Std, Median & N/A  \\ 
  \midrule
  Parameter Study  & AG-News Sampled & AnglE  & Min window: 50-200, Std mul: 0.5-3.0  & CV, P90/P10, Max/Min ratio      & N/A \\ 
\bottomrule
\end{tabular}}
\end{table*}

\subsection{Detailed Experiment Settings}
\label{appendix:expt-details:settings}

To comprehensively evaluate \textbf{SAFARI} and the effectiveness of \textbf{SFSes}, we conduct a series of experiments across diverse datasets and tasks, summarized in Table \ref{tab:expt-settings}.

\paragraph{Hierarchical Structure Discovery}
We apply \textbf{SAFARI} in a fully unsupervised setting across both text and image modalities.
%%%
For the \textbf{text} modality, we use the AG-News dataset with four categories. 
To reduce noise from common entities (e.g., \term{Reuters}) that appear across all categories, we retain only those unique to each class.
%%%
For the \textbf{image} modality, we use MIT-States, comprising 53,000 images across 245 object and 115 attribute classes. 
Four-level semantic hierarchies (Lv0 to Lv3) are generated for both datasets using Claude 3.7 Sonnet, as detailed in Appendix~\ref{appendix:expt-details:hierarchy}.

%%% metric
We evaluate the discovered semantic structures using the \emph{impurity} metric that measures label heterogeneity within clusters, ranging from 0 (perfect homogeneity) to higher values (increased mixing). 
%%%
Successful hierarchy discovery is indicated by both increasing diversity values with iterations and a consistent Lv0 $>$ Lv1 $>$ Lv2 $>$ Lv3 ordering that mirrors the progression from specific to abstract concepts.

\paragraph{Classification}
\textbf{SFSes} are constructed using class-labeled training data.
Test samples are classified by computing cosine distances to all \textbf{SFSes} and assigning the label of the nearest subspace.
%%%
For text classification, we use all embedding dimensions weighted by singular values; for image classification, we use only the top 5\% of dimensions without weighting.
This distinction reflects the higher variability and noise (e.g., background features) in image embeddings, which make full-dimensional or weighted comparisons less effective.
%%% metrics and baselines
We report precision, recall, and F1 score, and compare against four standard baselines: Support Vector Machine (\textbf{SVM}), K-Nearest Neighbors (\textbf{KNN}), Random Forest (\textbf{RF}), and Multi-Layer Perceptron (\textbf{MLP}).
\begin{itemize}
  \item \textbf{Topic Classification:} We use four text datasets: AG-News, AAPD, IMDB, and Yelp. 
  Entity embeddings are derived using \textbf{BLINK} \citep{wu2020scalable} for entity linking, followed by TF-IDF-based filtering \citep{schutze2008introduction, leskovec2020mining} to retain the top 2,100 entities.

  \item \textbf{Image Classification:} We use the MIT-States dataset with \textbf{CLIP} \cite{radford2021learning} to extract image embeddings.
  We retain 97 object classes with at least 240 samples each to ensure class balance.
\end{itemize}

For both modalities, embeddings are split into 80\% training and 20\% testing.

\paragraph{Computational Efficiency}
We apply \textbf{SAFARI} to the sampled AG-News dataset
%%% metrics
and compare the runtime of our approximate \textbf{Semantic Shift} computation against the full SVD-based variant. 
We report runtime and error rates to quantify the trade-off between efficiency and accuracy.

%%% Political Bias Detection: datasets
\paragraph{Political Bias Detection}
To assess \textbf{SAFARI}'s capability in capturing nuanced ideological distinctions beyond surface-level topics, we conduct a political bias detection task using the \textbf{NewsSpectrum} dataset. 
The dataset contains articles classified into five political categories (Left, Lean Left, Center, Lean Right, and Right), with labels derived from media source affiliations rather than individual content analysis.
%%% model
Articles are embedded using the AnglE~\citep{li2024aoe} sentence transformer, and the model is evaluated in a supervised classification setup with 80\%/20\% train-test split. 
%%% metrics
Evaluation metrics include F1 score and training time, and \textbf{SAFARI} is again compared with the same four baselines as the classification task.

%%% Component Analysis
\paragraph{Component Analysis of Semantic Shift}
To understand the internal mechanisms of \textbf{Semantic Shift}, we conduct an experiment to assess the effectiveness of its two components: Dimensional Importance Shift (DIS) and Directional Change (DC) on the sampled \textbf{AG-News} dataset. 
%%%
To quantify their relative contributions, we report their mean, median, and standard deviation values. 

\paragraph{Parameter Study}
Finally, we perform a parameter study on the \textbf{SAFARI}'s dynamic thresholding mechanism. 
We vary the minimum window size and test a range of standard deviation multipliers (0.5 to 3.0, with a step of 0.5). 
To evaluate the uniformity of detected semantic shifts, we report two metrics: Coefficient of Variation (CV) and Max/Min ratio.

\section{Political Bias Detection}
\label{appendix:political_bias}

Beyond conventional classification in Section 5.3, we explore the capability of \textbf{SAFARI} in detecting more abstract semantic patterns through political bias detection. 
Unlike topic classification, where semantic differences are often explicit and content-driven, political bias manifests in subtle linguistic choices and ideological framing that transcend specific topics. 
This presents a more challenging test for our method: \textbf{can SFS identify and represent these nuanced semantics that shape political orientation?} 

\paragraph{Experimental Setup}
We employ \textbf{NewsSpectrum} to assess whether SFSes capture abstract semantic patterns that reflect ideological perspectives.
Articles are categorized into five bias groups: Left, Lean Left, Center, Lean Right, and Right, though these labels are inherently imprecise as they are assigned based on media sources rather than content.
%%%
We use AnglE embeddings~\citep{li2024aoe}, which achieves state-of-the-art performance in the MTEB benchmark \citep{muennighoff2023mteb}. Since AnglE is a sentence-transformer model, these vectors represent articles rather than entities.

Following the setup in Section 5.3, we compare \textbf{SAFARI} against four standard classifiers: \textbf{SVM}, \textbf{KNN}, \textbf{RF}, and \textbf{MLP}, using 80\% of the data for training and 20\% for testing.
The results are displayed in Figures \ref{fig:pl_bias} and \ref{fig:pl_bias_all}.

\begin{figure}[t]
\centering
\subfigure[F1 Scores.]{
  \label{fig:pl_bias:f1}
  \includegraphics[width=0.48\columnwidth]{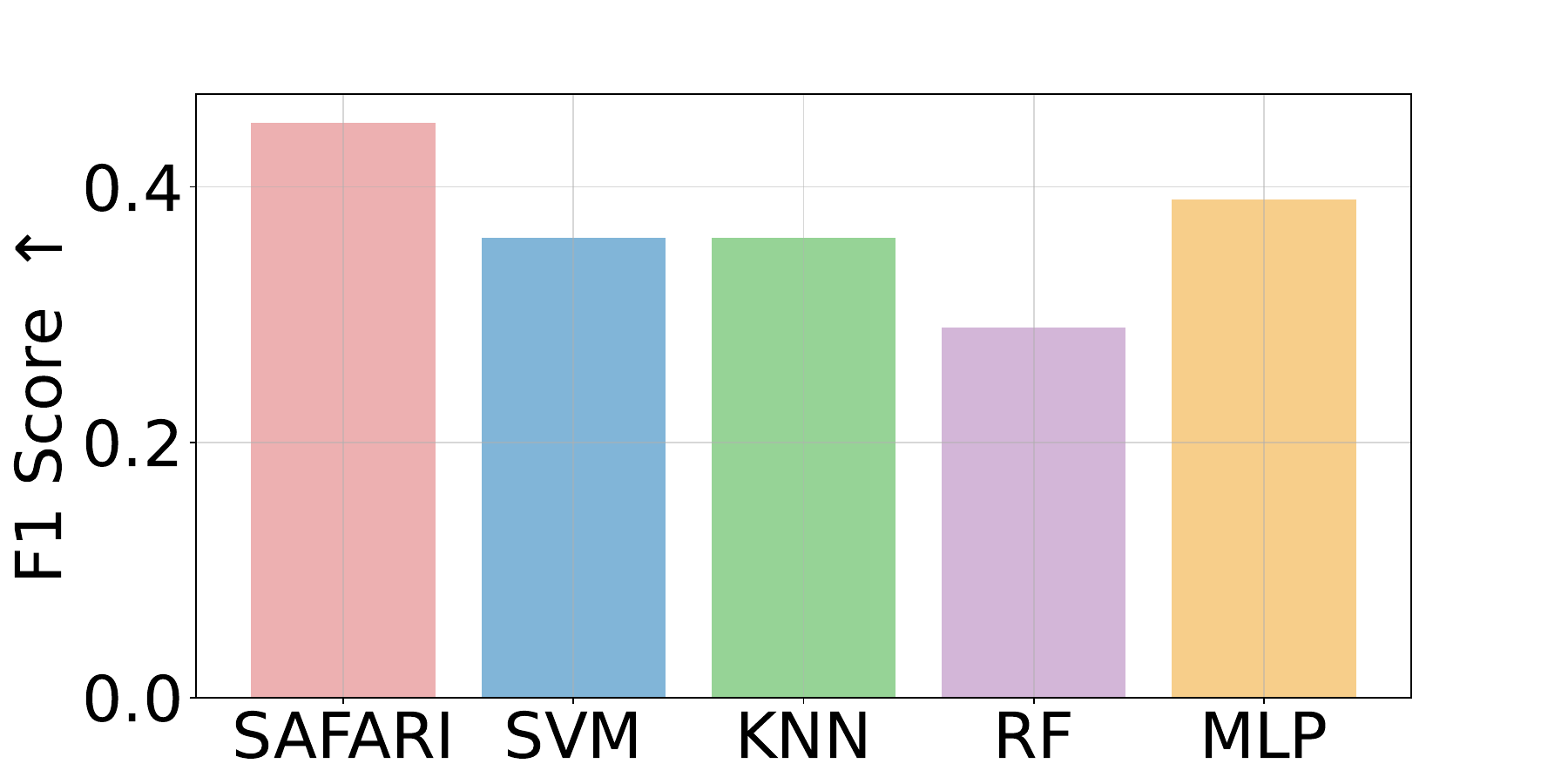}}
\subfigure[Training Time.]{
  \label{fig:pl_bias:time}
  \includegraphics[width=0.48\columnwidth]{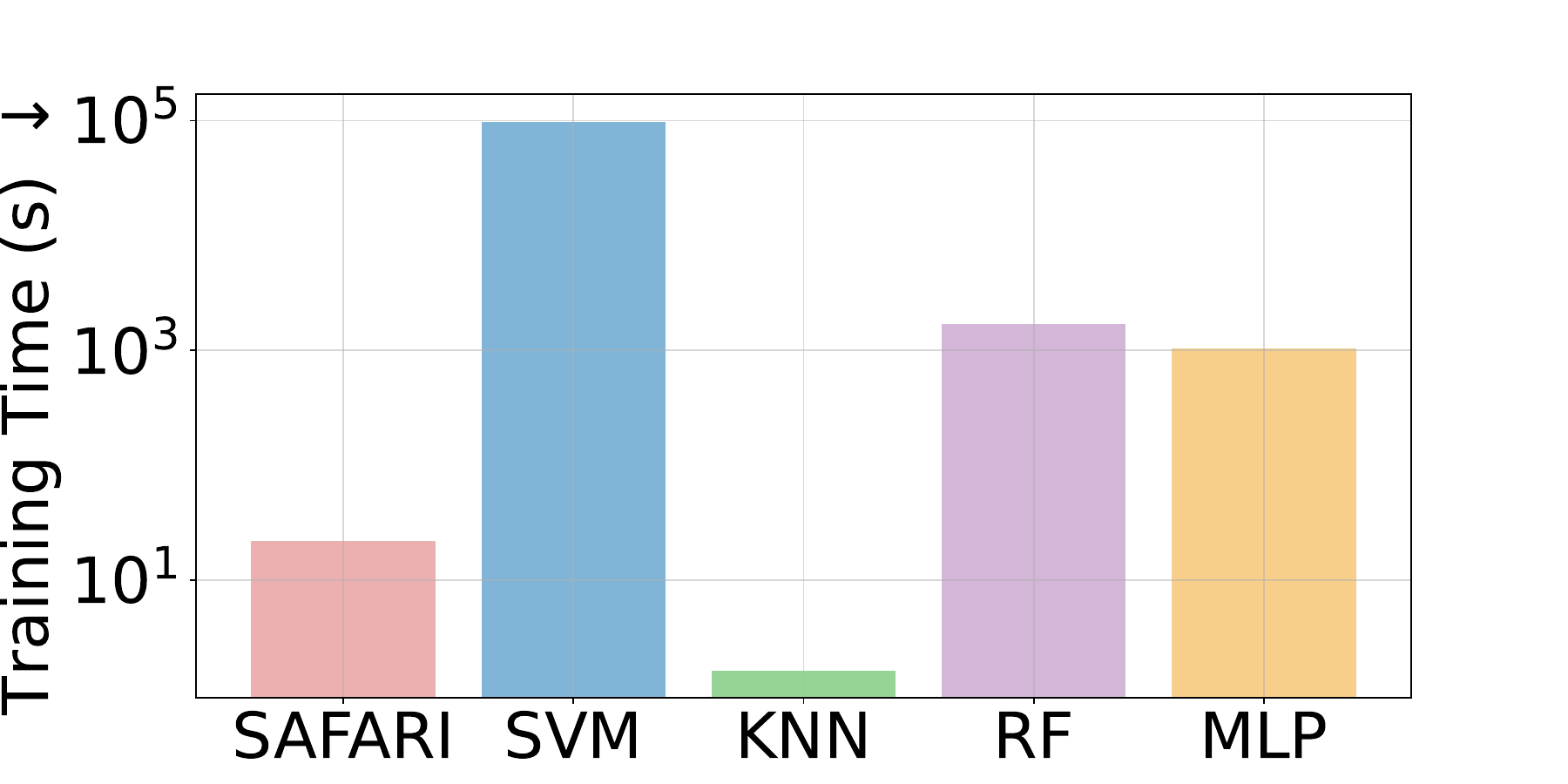}}
\caption{Political bias detection results.}
\label{fig:pl_bias}
\end{figure}

\begin{figure*}[t]
  \centering
  \includegraphics[width=0.99\textwidth]{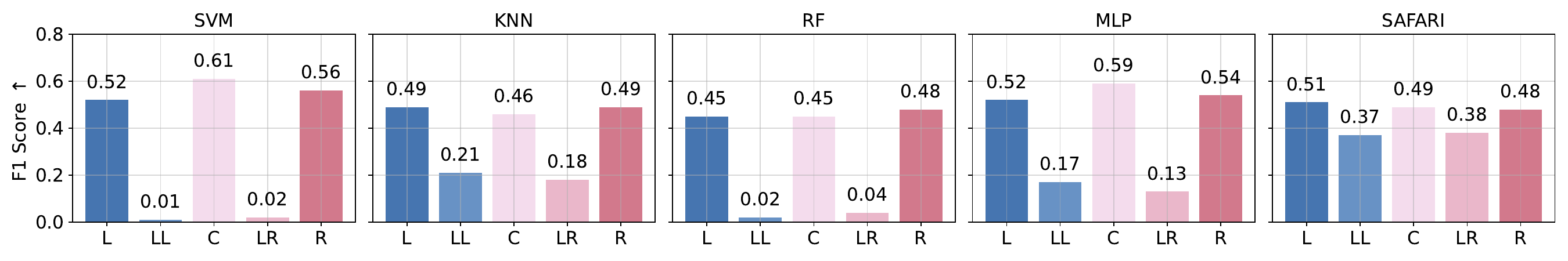}
  \caption{F1 scores across five political bias categories: L (Left), LL (Lean Left), C (Center), LR (Lean Right), and R (Right).}
  \label{fig:pl_bias_all}
\end{figure*}

\begin{figure*}[t]
  \centering
  \includegraphics[width=0.8\textwidth]{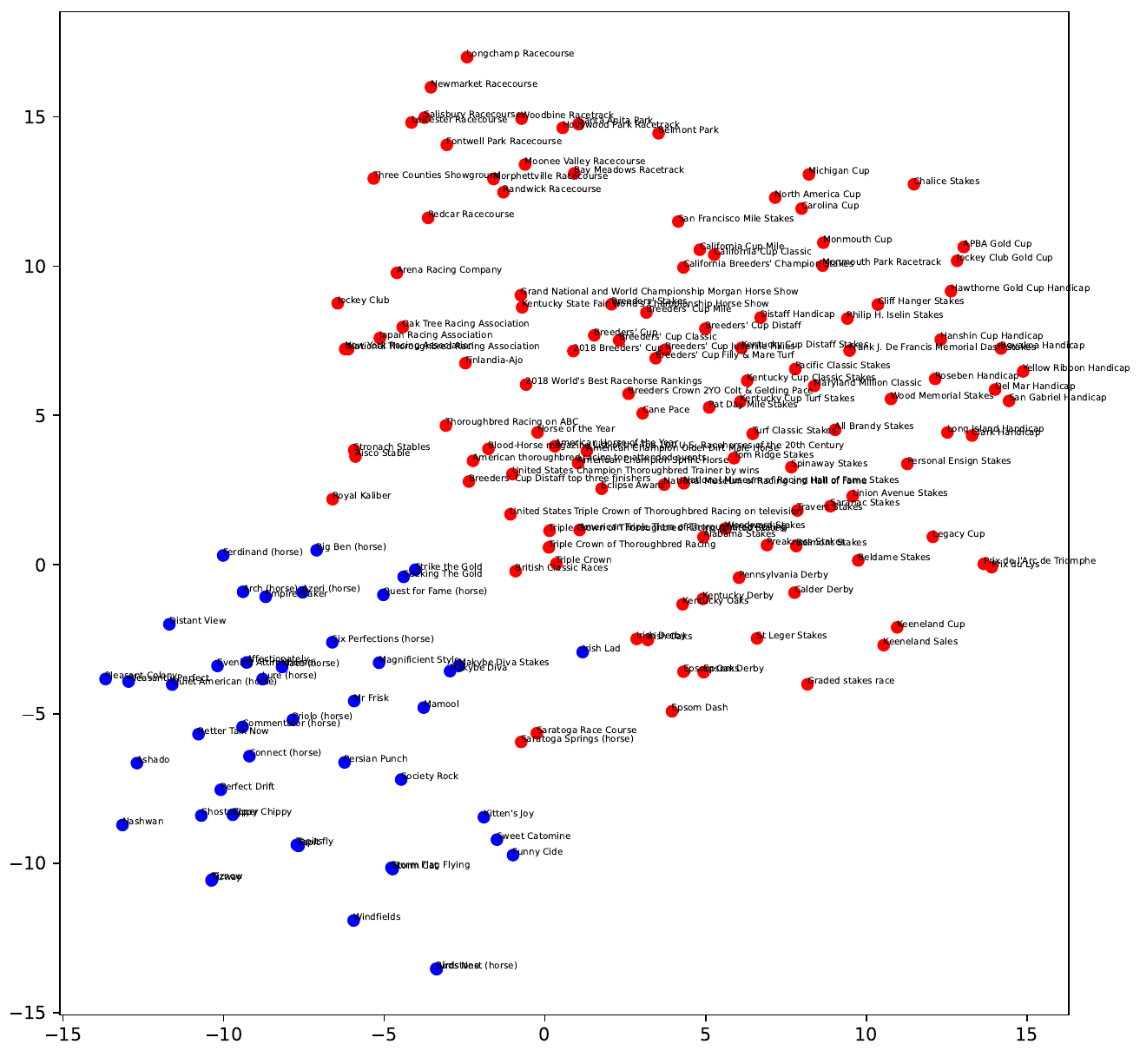}
  \caption{Famous racing horses merged with other horse racing entities.} 
  \label{fig:exp3_horse_racing}
\end{figure*}

\paragraph{Results and Analysis}
The results in Figure \ref{fig:pl_bias} highlight that \textbf{SAFARI} achieves the best balance between classification performance and computational efficiency for political bias detection. 
%%%
\textbf{SAFARI} achieves the highest F1 score (0.45) while maintaining a low training time (21.9 seconds), making it the most practical choice. 
In comparison, \textbf{MLP} reaches a slightly lower F1 score (0.39) but incurs a significantly higher training time (1,033.28 seconds). 
Other classifiers exhibit varying efficiency-performance trade-offs: \textbf{SVM} and \textbf{KNN} both have an F1 score of 0.36, yet \textbf{SVM} demands extensive training time (27.29 hours), whereas \textbf{KNN} is the fastest (1.64 seconds) but offers weaker classification accuracy. 
\textbf{RF} performs the worst, with an F1 score of 0.29 and a substantial training time of 1,684.05 seconds.
%%%
These results underscore the effectiveness of \textbf{SAFARI}, making it scalable and practical for political bias detection tasks.

Figure \ref{fig:pl_bias_all} further validates \textbf{SAFARI}'s superior and well-balanced performance in political bias detection compared to standard classifiers. 
\textbf{SAFARI} maintains stable F1 scores across all five categories (ranging from 0.37 to 0.51), whereas standard classifiers exhibit substantial fluctuations and consistently poor performance in lean positions.
Notably, \textbf{SVM} and \textbf{RF} achieve only 0.01--0.04 F1 scores for Lean Left and Lean Right categories, with only \textbf{KNN} performing slightly better at 0.21.
%%%
Furthermore, while some standard classifiers, such as \textbf{SVM} and \textbf{MLP}, show relatively higher F1 scores for center classification, they struggle significantly with Lean Left and Lean Right classifications.
In contrast, \textbf{SAFARI} delivers robust and consistent performance across the entire political spectrum. 
This advantage stems from its unique approach of using \textbf{SFSes} to represent different political categories. 

Unlike standard classifiers with rigid decision boundaries, \textbf{SAFARI} uses subspaces whose basis vectors naturally encode both shared and category-specific semantic patterns.
%%%
Common linguistic structures or ideological overlaps between categories emerge as shared directions within the subspaces, while category-specific traits are preserved in distinct dimensions.
%%%
This enables \textbf{SAFARI} to recognize when an article aligns semantically with multiple political categories, supporting nuanced, context-aware classification across the political spectrum.

\section{Component Analysis of Semantic Shift}
\label{appendix:dis-dc}

\paragraph{Experimental Setup}
To dissect the inner workings of Semantic Shift, we conduct a controlled experiment that quantifies the individual contributions of its two components: \textbf{Dimensional Importance Shift (DIS)} and \textbf{Directional Change (DC)}. 
%%%
Our analysis is based on 9,778 valid samples drawn from \textbf{SAFARI}'s clustering process, which reveals distinct patterns in their respective impacts.
The results are presented in Table \ref{tab:dis-dc-stats}.

\begin{table}[t]
\centering
\small
\renewcommand{\arraystretch}{1.0}
\caption{Statistics of the contributions of the DIS and DC components.}
\label{tab:dis-dc-stats}
\begin{tabular}{lcc}
  \toprule
  \textbf{Metric} & \textbf{DIS} & \textbf{DC}\\
  \midrule
  Mean & 5.60 & 3.19\\
  Median & 3.97 & 2.36\\
  Standard Deviation & 7.41 & 3.25\\
  \midrule
  Log Contribution & 59.3\% & 40.7\%\\
  \bottomrule
\end{tabular}
\end{table}

\paragraph{Results and Analysis}
Our findings indicate that DIS (5.60) is the dominant factor in stability measurement, contributing approximately 1.8$\times$ more than DC (3.19) on average. 
The right-skewed distribution, evidenced by lower median values compared to means, indicates the presence of significant outliers where DIS is dramatically more influential. 
%%%
While DIS (59.3\%) demonstrates greater overall impact, accounting for nearly 60\% of the effect, DC (40.7\%) remains a substantial contributor at approximately 40\%. 
This confirms that both mechanisms play significant roles in the stability dynamics, though with DIS exerting the stronger influence in most scenarios.

\section{Parameter Study about the Dynamic Threshold Mechanism in SAFARI}
\label{appendix:dyn-threshold}

To robustly detect semantic transitions throughout clustering, \textbf{SAFARI} employs a dynamic thresholding mechanism based on a sliding window. 
%%%
Rather than using a fixed window size across all iterations, it applies a recursive divide-and-conquer strategy: the sequence of semantic shifts is iteratively split into smaller segments based on distributional imbalance between halves. 
This process continues until each segment either reaches a predefined minimum size or exhibits sufficiently balanced distribution.
%%%
This adaptive strategy prevents extreme Semantic Shifts in later iterations from overshadowing earlier, subtler transitions, ensuring more balanced detection across the entire clustering process.

We study how two parameters affect the behavior of this dynamic threshold mechanism: (1) the \textbf{Standard Deviation Multiplier (SDM)} and (2) the \textbf{Minimum Window Size (MWS)} used to define the dynamic threshold.
%%%
We evaluate the uniformity of detected semantic shifts across different settings to identify configurations that balance sensitivity to meaningful changes with temporal consistency.

\paragraph{Experimental Setup} 
To evaluate the uniformity of Semantic Shifts, we use two complementary metrics: 
(1) \textbf{Coefficient of Variation (CV)}: measures dispersion relative to the mean; lower values imply greater uniformity.
(2) \textbf{Max/Min Ratio}: captures the full range of detected shift magnitudes; smaller values indicate better balance across iterations.
%%%
Results are shown in Table \ref{tab:std_multiplier}.

\begin{table}[t]
\centering
\small
\renewcommand{\arraystretch}{1.0}
\caption{Performance metrics across standard multipliers.}
\label{tab:std_multiplier}
% \resizebox{0.99\columnwidth}{!}{%
\begin{tabular}{ccccc}
  \toprule
  \textbf{SDM} & \textbf{MWS} & \textbf{CV $\downarrow$} & \textbf{Max/Min Ratio $\downarrow$} \\
  \midrule
  3.0 & 50-200 & 3.75 & 1,540.6 \\
  2.5 & 50-200 & 4.42 & 1,811.3 \\
  2.0 & 50-200 & 5.10 & 2,193.1 \\
  1.5 & 50-200 & 5.84 & 2,778.6 \\
  1.0 & 50-200 & 6.83 & 3,599.7 \\
  0.5 & 50-200 & 6.78 & 4,384.8 \\
  \bottomrule
\end{tabular}%}
\end{table}

\paragraph{Results and Analysis} 
The configuration with a standard deviation multiplier of 3.0 consistently achieves the most uniform distribution of Semantic Shifts (CV = 3.75, Max/Min ratio = 1,540.6), regardless of the minimum window size.
%%%
This setting still ensures that no single phase of the clustering process dominates the detection of \textbf{SFSes}.
%%%
Overall, our parameter study reveals that achieving uniform semantic shift distribution depends primarily on the standard deviation multiplier rather than the minimum window size configuration. 
The multiplier of 3.0 effectively prevents later iterations from dominating the analysis while maintaining sensitivity to meaningful semantic changes throughout the clustering process.

\section{More Analysis on Hierarchical Structure}
\label{appendix:more-hierarchy}

We explore the differences between the hierarchical semantic structures identified by \textbf{SAFARI} in embedding spaces and the more intuitive hierarchies found in natural human language.
%%%
In human language, semantics typically follow a logical hierarchy, progressing from specific, concrete entities to more abstract concepts, much like an ontology. 
%%%
However, in embedding spaces, this progression is not always intuitive. 
The distinction between \emph{specific} and \emph{abstract} depends more on the data and model than on human reasoning, often leading to groupings that diverge from what we would expect based on natural language understanding.

For example, as illustrated in Figure 7, USA basketball teams are first grouped with USA football teams, and later, sports teams from various locations are merged, as shown in Figure 8. 
This follows a logical hierarchical structure, from more specific categories to broader ones. 
%%%
Yet, as shown in Figure \ref{fig:exp3_horse_racing} (at iteration 19,790), entities such as horse racing clubs, companies, and events (e.g., \term{Jockey Club}) are merged with famous racing horses. 
This merging of horse racing happens thousands of iterations after the merging of football and basketball teams in the USA. Following the ontology-like progression, we would expect more abstract concepts.  However, horse racing is not a more abstract concept compared with other sports.

These examples illustrate that the hierarchical structures emerging from embedding spaces are governed by the model's learned representations rather than human-designed logic.
While some align intuitively with natural semantic categories, others can be surprising--revealing how models encode relationships that reflect statistical regularities in the data rather than explicit reasoning.
This underscores the need for cautious interpretation: semantic hierarchies derived from embeddings may not faithfully mirror human conceptual structures and should be analyzed with an awareness of the model's inductive biases.

% \newpage
% \input{10_rebuttal}

\end{document}